\pdfoutput=1
\documentclass[colorlinks,dvipsnames]{article}

\PassOptionsToPackage{numbers}{natbib}
\usepackage[preprint]{neurips_2021}

\usepackage[utf8]{inputenc} %
\usepackage[T1]{fontenc}    %
\usepackage{hyperref}       %
\usepackage{url}            %
\usepackage{booktabs}       %
\usepackage{amsfonts}       %
\usepackage{nicefrac}       %
\usepackage{microtype}      %
\usepackage{xcolor}         %
\usepackage[numbers]{natbib}
\usepackage{wrapfig}

\usepackage{todonotes}
\usepackage{amsthm}
\usepackage{algorithm}
\usepackage{algorithmic}
\usepackage{siunitx}
\usepackage{thmtools,thm-restate}
\usepackage{enumitem}
\usepackage{comment}
\newtheorem{proposition}{Proposition}
\newtheorem{lemma}[proposition]{Lemma}

\usepackage{amsmath,amsfonts,bm}

\def\Secref#1{Section~\ref{#1}}

\def\1{\bm{1}}

\def\rvf{{\mathbf{f}}}

\def\vmu{{\bm{\mu}}}
\def\vtheta{{\bm{\theta}}}
\def\vphi{{\bm{\phi}}}

\def\ve{{\bm{e}}}

\def\vm{{\bm{m}}}

\def\vw{{\bm{w}}}
\def\vx{{\bm{x}}}
\def\vy{{\bm{y}}}
\def\vz{{\bm{z}}}

\def\mA{{\bm{A}}}
\def\mB{{\bm{B}}}
\def\mC{{\bm{C}}}

\def\mH{{\bm{H}}}

\def\mU{{\bm{U}}}
\def\mV{{\bm{V}}}
\def\mW{{\bm{W}}}
\def\mX{{\bm{X}}}
\def\mY{{\bm{Y}}}

\def\mSigma{{\bm{\Sigma}}}

\DeclareMathAlphabet{\mathsfit}{\encodingdefault}{\sfdefault}{m}{sl}
\SetMathAlphabet{\mathsfit}{bold}{\encodingdefault}{\sfdefault}{bx}{n}

\newcommand{\R}{\mathbb{R}}

\newcommand{\sigmoid}{\sigma}

\DeclareMathOperator*{\argmax}{arg\,max}

\usepackage{authblk}

\definecolor{mydarkblue}{rgb}{0,0.08,0.45}

\usepackage{hyperref}
\hypersetup{
    urlcolor=MidnightBlue, 
    linkcolor=MidnightBlue,
    citecolor=MidnightBlue,
}
\usepackage{cleveref}
\usepackage{pgfplots}
\pgfplotsset{compat=newest}
\usepgfplotslibrary{groupplots}
\usepgfplotslibrary{dateplot}

\usepackage{tikz}

\DeclareUnicodeCharacter{2212}{-}

\makeatletter
\renewcommand{\todo}[2][]{\tikzexternaldisable\@todo[#1]{#2}\tikzexternalenable}
\makeatother

\crefname{lemma}{lemma}{lemmas}
\Crefname{lemma}{Lemma}{Lemmas}
\crefname{theorem}{theorem}{theorems}
\Crefname{theorem}{Theorem}{Theorems}
\crefname{proposition}{proposition}{propositions}
\Crefname{proposition}{Proposition}{Propositions}

\title{Mixtures of Laplace Approximations \\ for Improved \textit{Post-Hoc} Uncertainty in Deep Learning}

\author[t]{Runa Eschenhagen\thanks{Correspondence to: \texttt{runa.eschenhagen@student.uni-tuebingen.de}.}$\;^{,}$}
\author[c,m]{Erik Daxberger}
\author[t,m]{Philipp Hennig}
\author[t]{Agustinus Kristiadi}
\affil[t]{University of T\"{u}bingen}
\affil[c]{University of Cambridge}
\affil[m]{MPI for Intelligent Systems, T\"{u}bingen}

\begin{document}

\maketitle
\setcounter{footnote}{0}

\begin{abstract}
Deep neural networks are prone to overconfident predictions on outliers.
Bayesian neural networks and deep ensembles have both been shown to mitigate this problem to some extent.
In this work, we aim to combine the benefits of the two approaches by proposing to predict with a Gaussian mixture model posterior that consists of a weighted sum of Laplace approximations of independently trained deep neural networks.
The method can be used \textit{post hoc} with any set of pre-trained networks and only requires a small computational and memory overhead compared to regular ensembles.
We theoretically validate that our approach mitigates overconfidence ``far away'' from the training data and empirically compare against state-of-the-art baselines on standard uncertainty quantification benchmarks.
\end{abstract}

\section{Introduction}
\label{short_intro}

While deep neural networks (DNNs) have achieved impressive results in a wide range of domains, they are prone to make overconfident predictions on outliers, such as \textit{out-of-distribution} (OOD) data and data under \textit{distribution shift} \citep{datasetshit2009, Nguyen2015DeepNN}; this is especially harmful in safety-critical applications \citep{amodei2016concrete}.

Bayesian inference of the DNN weights, resulting in a class of methods called \emph{Bayesian neural networks} (BNNs), is a principled approach for quantifying predictive uncertainty. 
Commonly, a Gaussian distribution is used to approximate the true posterior. Many methods have been proposed to infer the parameters of this distribution, such as Laplace approximations \citep{mackay1992practical, Ritter2018ASL}, variational inference (VI) \citep{Graves2011PracticalVI, blundell2015weight}, and sampling-based approaches \citep{maddox2019simple,zhang2019cyclical}.
While BNNs have been demonstrated to scale to large networks and problems like ImageNet \citep{osawa2019practical, maddox2019simple}, recent work \cite{Ovadia2019CanYT, Ashukha2020PitfallsOI} has shown that they tend to be outperformed by a simpler method called Deep Ensemble \citep{Lakshminarayanan2016SimpleAS}, which averages the predictions of multiple identical networks, independently trained with different random initializations.

In this work, we combine---intuitively speaking---the \textit{global} uncertainty of Deep Ensembles with the \textit{local} uncertainty around a single mode of BNNs with Gaussian approximate posterior, in a \emph{post-hoc} way (\Cref{fig:one}); this idea has been proposed previously \cite{Wilson2020BayesianDL}.
We apply Laplace approximations to multiple independent pre-trained DNNs; specifically, we focus on last-layer Laplace approximations \cite{SnoekRSKSSPPA15, Kristiadi2020BeingBE}: this makes our method fast (see \Cref{tab:costs}), scalable, and easy to implement without sacrificing predictive and uncertainty quantification performance. We call the resulting method \emph{Mixtures of Laplace Approximations} (MoLA). 
It is trivial to implement in PyTorch, using the recently published \texttt{\textcolor{mydarkblue}{laplace}} library\footnote{Available at \url{https://github.com/AlexImmer/Laplace}.} \cite{laplace}.
Besides empirically studying our method, we also extend the analysis of \citet{Kristiadi2020BeingBE} on single-Gaussian posteriors to the multi-class classification setting and subsequently show that MoG-based posteriors also avoid overconfidence ``far away'' from the training data.

\begin{figure}
    \centering

    \includegraphics[width=\textwidth]{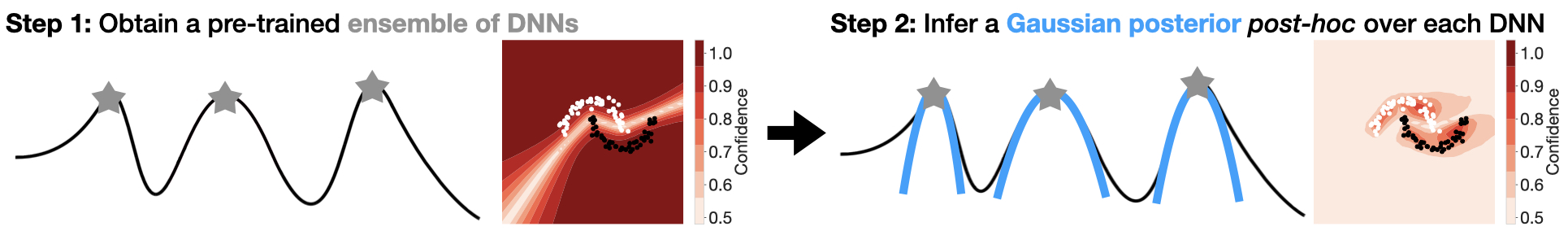}

    \caption{A Deep Ensemble captures \textit{global} uncertainty by using only a small number of point estimates at modes (represented by stars) of the posterior, resulting in overconfident predictions away from the training data. Here we construct \emph{post-hoc} \emph{local} uncertainty (blue curves) around each mode, resulting in improved predictive uncertainty. Figure adapted from \citet{Fort2019DeepEA}.}
    \label{fig:one}
\end{figure}

\section{Background}
\label{sec:short_background}

We focus on the multi-class classification setting of a dataset $\mathcal{D} = \{\vx_i, \vy_i\}_{i=1}^N$ in $C$ classes, with i.i.d.~data points $\vx_i \in \mathbb{R}^D$ and the corresponding one-hot encoded labels $\vy_i \in \{\ve_1, ..., \ve_C\},$ where $\ve_i$ is the $i$-th $C-$dimensional standard unit vector.
Consider a neural network $\rvf_{\vtheta}: \mathbb{R}^D \rightarrow \mathbb{R}^C$ parameterized by $\vtheta \in \mathbb{R}^M$ with posterior distribution $p(\vtheta|\mathcal{D}) \propto p(\mathcal{D}|\vtheta) p(\vtheta).$
A common way to train such a network is via \textit{maximum a posteriori} (MAP) inference, i.e.\ we find the MAP estimate $\vtheta_{\text{MAP}} = \argmax_{\vtheta} p(\vtheta|\mathcal{D}) = \argmax_{\vtheta} \log p(\mathcal{D}|\vtheta) + \log p(\vtheta).$
The log-likelihood $\log p(\mathcal{D}|\vtheta)$ is equivalent to the negative loss function, e.g. the cross-entropy loss for classification, and the prior distribution $p(\vtheta)$ over model parameters $\vtheta$ is usually a simple isotropic Gaussian $\mathcal{N}(\vtheta|\mathbf{0}, \lambda^{-1}\mathbf{I})$ with prior precision $\lambda \in \mathbb{R_{+}}$, which is closely related to $L_2$ regularization and weight decay. 

\subsection{(Last-Layer) Laplace Approximation}

A standard way to define the approximate posterior $q(\vtheta | \mathcal{D})$ is via a simple Gaussian approximation $\mathcal{N}(\vtheta|\vmu, \mSigma)$. One option is to employ the Laplace approximation \cite{mackay1992practical}.
In this technique, we apply a second-order Taylor expansion at $\vtheta_{\text{MAP}}$ to the log joint distribution, and then choose $\vmu = \vtheta_{\text{MAP}}$ and $\mSigma = (\mH_\vtheta + \lambda \mathbf{I})^{-1} \in \mathbb{R}^{M \times M},$ where $\mH_\vtheta$ is the Hessian of $-\log p(\mathcal{D}|\vtheta)$ w.r.t. $\vtheta$ evaluated at $\vtheta_{\text{MAP}}$.

It has been shown that last-layer Bayesian approximations yield a competitive performance compared to all-layer ones \citep{SnoekRSKSSPPA15,ober2019benchmarking,brosse2020last}. In this case, we consider the network $\rvf_\vtheta$ as a linear function in the weight matrix $\mW \in \mathbb{R}^{C \times P}$ of the last layer, i.e. $\rvf_{\mW}(\vx_*) = \mW \vphi(\vx_*),$ with \emph{fixed} features $\vphi(\vx_*) =: \vphi_* \in \mathbb{R}^P$, which are simply the output of the penultimate layer of the neural network given an arbitrary input $\vx_*$. 
Here, we only need to do a Laplace approximation on the weight matrix $\mW$: we obtain a Gaussian approximation $\mathcal{N}(\mathrm{vec}(\mW) | \vmu, \mSigma)$ with $\vmu = \text{vec}(\mW_{\text{MAP}}) \in \mathbb{R}^{CP}$ and $\mSigma = (\mH_\mW + \lambda \mathbf{I})^{-1} \in \mathbb{R}^{CP \times CP},$ where $\mH_\mW$ is the Hessian of $-\log p(\mathcal{D}|\mW)$ w.r.t. $\text{vec}(\mW)$ evaluated at $\text{vec}(\mW_{\text{MAP}})$. This method is called \emph{last-layer Laplace approximation} \citep[LLLA,][]{Kristiadi2020BeingBE}. Since $\rvf_{\mW}$ is linear in $\mW,$ we also have a Gaussian distribution $p(\rvf_* | \vx_*, \mathcal{D}) = \mathcal{N}(\rvf_* |\vm_*, \mC_*)$ over the network outputs $\rvf_* := \rvf(\vx_*)$, where $\vm_* = \mW_{\text{MAP}} \vphi_* \in \mathbb{R}^C$ and $\mC_* = (\vphi_*^\top \otimes \mathbf{I}) \mSigma (\vphi_* \otimes \mathbf{I}) \in \mathbb{R}^{C \times C}$.\footnote{$\otimes$ denotes the Kronecker product.}

\section{Mixtures of Laplace Approximations}

Consider DE's MAP estimates $(\vtheta_\text{MAP}^{(k)})_{k=1}^K$ of an $L$-layer network. For each $k = 1 \dots, K$, we treat the first $L-1$ layers as a fixed feature map $\vphi^{(k)}$ and construct a Gaussian approximate posterior $\mathcal{N}(\mathrm{vec}(\mW) | \vmu^{(k)}, \mSigma^{(k)})$ over the last-layer weights via LLLA. That is, we set $\vmu^{(k)} = \mathrm{vec}(\mW^{(k)}_\text{MAP})$ and $\mSigma^{(k)}$ to be the inverse Hessian of the negative log-posterior w.r.t. $\mathrm{vec}(\mW)$ at $\vmu^{(k)}$. Given a sequence $\boldsymbol{\pi} := (\pi^{(k)})_{k=1}^K$ of non-negative real numbers with $\sum_{k=1}^K \pi^{(k)} = 1$, we define the approximate posterior as\footnote{The choice of $\boldsymbol{\pi}$ shall be discussed in \Cref{subsec:practical}.} $p_\text{MoLA}(\mathrm{vec}(\mW) | \mathcal{D}) := \sum_{k=1}^K \pi^{(k)} \, \mathcal{N}(\mathrm{vec}(\mW) | \vmu^{(k)}, \mSigma^{(k)}).$
By employing the \emph{multi-class probit approximation} \citep[MPA,][]{Gibbs98, lu2021meanfield}, MoLA's predictive distribution takes a particularly simple form:
\begin{equation} \label{eq:mola_predictive_main}
    p_\text{MoLA}(\vy = \ve_c | \vx_*, \mathcal{D}) \approx \sum_{k=1}^K \pi^{(k)} \, \sigma(\vz_*^{(k)})_c ,
    \end{equation}
where $\vz_*^{(k)}$ is a vector with $i$-th component $\vz_{*i}^{(k)} = \vm_{*i}^{(k)} / \sqrt{1 + (\pi/8) \mC_{*ii}^{(k)}}$ for each $k = 1, \dots, K$ (see \Cref{sec:proofs} for the derivation), and $\sigma(\cdot)$ is the softmax function. The MPA enables fast predictions with a single forward pass per component. While we focus on MoLA applied to a regular DE, recent work of \citet{Havasi2021mimo} allows us to apply MoLA to a \emph{single} DNN (MIMO-MoLA). This further improves the efficiency of MoLA without sacrificing much performance (see \Cref{fig:main_exp}). For more details on practical considerations when applying MoLA, please refer to \Cref{subsec:practical}. 

In \Cref{subsec:analysis} we show how MoLA can mitigate overconfident predictions of DNNs with ReLU nonlinearities  ``far away'' from the training data, in the sense that a training input $\vx$ is scaled with a scalar $\delta > 0$, and as $\delta \rightarrow \infty$ \cite{Hein_2019_CVPR}; all proofs are in \Cref{sec:proofs}.

\section{Experiments}
\label{sec:experiments}

\begin{table}[t]
    \caption{In-distribution: CIFAR-10 $\rightarrow$ OOD. Values are means along with their standard errors
    over five runs with models trained with different random initalizations.
    }
    \label{tab:cifar10_ood_main}

    \centering
    \fontsize{8}{10}\selectfont

    \begin{sc}
        \begin{tabular}{lccccccc}
            \toprule
            & In-Dist. & \multicolumn{2}{c}{SVHN} & \multicolumn{2}{c}{LSUN} & \multicolumn{2}{c}{CIFAR-100} \\
            Method & MMC & MMC $\downarrow$ & AUROC $\uparrow$ & MMC $\downarrow$ & AUROC $\uparrow$ & MMC $\downarrow$ &  AUROC $\uparrow$ \\
            \midrule
            MAP  &  97.2 $\pm$ 0.0 &  77.5 $\pm$ 2.9 &  91.7 $\pm$ 1.2 &  71.7 $\pm$ 0.8 &  94.3 $\pm$ 0.3 &  79.2 $\pm$ 0.1 &  90.0 $\pm$ 0.1 \\
            DE &  96.1 $\pm$ 0.0 &  62.8 $\pm$ 0.7 &  95.4 $\pm$ 0.2 &  59.2 $\pm$ 0.5 &  96.0 $\pm$ 0.1 &  70.7 $\pm$ 0.1 &  91.3 $\pm$ 0.1 \\
            SWAG & 95.1 $\pm$ 0.4 & 69.3 $\pm$ 4.0 & 91.6 $\pm$ 1.3 &  62.2 $\pm$ 2.3 & 94.0 $\pm$ 0.7 & 73.0 $\pm$ 0.4 & 88.2 $\pm$ 0.5 \\
            MSWAG  &  94.5 $\pm$ 0.2 &  57.0 $\pm$ 1.2 &  95.6 $\pm$ 0.5 &  56.3 $\pm$ 1.0 &  95.6 $\pm$ 0.3 &  65.5 $\pm$ 0.5 &  91.1 $\pm$ 0.1 \\
            LLLA &  94.1 $\pm$ 0.2 &  60.5 $\pm$ 4.0 &  93.6 $\pm$ 1.1 &  54.5 $\pm$ 1.5 &  95.4 $\pm$ 0.3 &  64.8 $\pm$ 0.6 &  90.8 $\pm$ 0.1 \\
            \midrule
            MoLA & 93.8 $\pm$ 0.1 &  \textbf{52.3 $\pm$ 0.7} &  \textbf{96.2 $\pm$ 0.2} &  \textbf{48.3 $\pm$ 0.5} &  \textbf{96.9 $\pm$ 0.1} &  \textbf{61.2 $\pm$ 0.2} &  \textbf{92.0 $\pm$ 0.0} \\
            \bottomrule
        \end{tabular}
    \end{sc}
\end{table}

\begin{figure*}[t]
    \centering

    \input{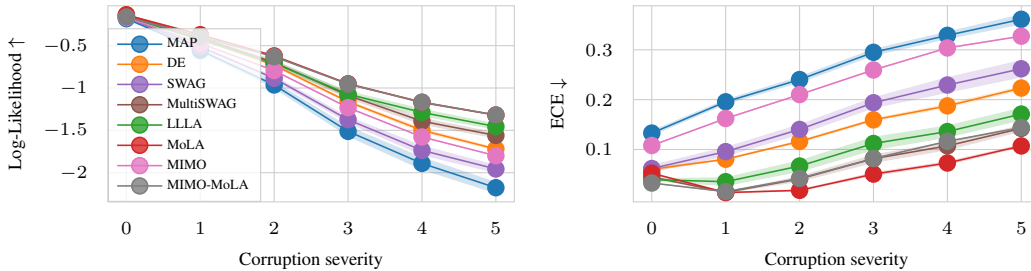}

    \vspace{-1em}
    \caption{All methods on corrupted CIFAR-10, in terms of the log-likelihood (\textbf{left}, higher is better) and expected calibration error (\textbf{right}, lower is better) metrics. Dots represent means while shades represent standard errors over five runs with models trained with different random initalizations.}%
    \label{fig:main_exp}
\end{figure*}

We conduct extensive experiments on multiple image classification benchmarks, such as rotated MNIST, corrupted CIFAR-10 (\Cref{fig:main_exp}), and OOD detection tasks with MNIST and CIFAR-10 (\Cref{tab:cifar10_ood_main}) as in-distribution datasets. Moreover, we consider corrupted ImageNet to demonstrate the scalability of our approach. Generally, MoLA matches or outperforms all other considered methods, including Multi-SWAG, despite being cheaper. See \Cref{sec:extended_experiments} for the results and a detailed discussion.

\section{Conclusion}

We propose to combine two complementary forms of approximate Bayesian inference to infer and predict with a mixture of Gaussians constructed from \textit{post-hoc} Laplace approximations. Empirically, the method compares favorably against state-of-the-art baselines on image classification benchmarks, in terms of performance as well as cost. The method can also be used with a single-model thanks to MIMO \cite{Havasi2021mimo}, which is especially attractive in the case where no pre-trained ensemble is available.

\clearpage

\begin{ack}
    We thank Alexander Immer for helpful discussions on using the marginal likelihood for model weighting.
    
    R.E., P.H., and A.K. gratefully acknowledge financial support by the European Research Council through ERC StG Action 757275 / PANAMA; the DFG Cluster of Excellence “Machine Learning - New Perspectives for Science”, EXC 2064/1, project number 390727645; the German Federal Ministry of Education and Research (BMBF) through the Tübingen AI Center (FKZ: 01IS18039A); and funds from the Ministry of Science, Research and Arts of the State of Baden-Württemberg.
    E.D. acknowledges funding from the EPSRC and Qualcomm.
    A.K. is grateful to the International Max Planck Research School for Intelligent Systems (IMPRS-IS) for support.
\end{ack}

{
\small
\bibliography{refs}
\bibliographystyle{abbrvnat}
}

\clearpage

\clearpage

\appendix

\appendix

\section{Background}
\label{sec:background}

\subsection{Bayesian Deep Learning}
\label{subsec:BDNN}

We focus on the multi-class classification setting of a dataset $\mathcal{D} = \{\vx_i, \vy_i\}_{i=1}^N$ in $C$ classes, with i.i.d.~data points $\vx_i \in \mathbb{R}^D$ and the corresponding one-hot encoded labels $\vy_i \in \{\ve_1, ..., \ve_C\},$ where $\ve_i$ is the $i$-th $C-$dimensional standard unit vector.
Consider a neural network $\rvf_{\vtheta}: \mathbb{R}^D \rightarrow \mathbb{R}^C$ parameterized by $\vtheta \in \mathbb{R}^M$ with posterior distribution $p(\vtheta|\mathcal{D}) \propto p(\mathcal{D}|\vtheta) p(\vtheta).$
A common way to train such a network is via \textit{maximum a posteriori} (MAP) inference, i.e.\ we find the MAP estimate
\begin{equation} \label{eq:map_estimation}
    \vtheta_{\text{MAP}} = \argmax_{\vtheta} p(\vtheta|\mathcal{D})
        = \argmax_{\vtheta} \log p(\mathcal{D}|\vtheta) + \log p(\vtheta) .
\end{equation}
In \eqref{eq:map_estimation}, the log-likelihood $\log p(\mathcal{D}|\vtheta)$ is equivalent to the negative loss function, e.g. the cross-entropy loss for classification, and the prior distribution $p(\vtheta)$ over model parameters $\vtheta$ is usually a simple isotropic Gaussian $\mathcal{N}(\vtheta|\mathbf{0}, \lambda^{-1}\mathbf{I})$ with prior precision $\lambda \in \mathbb{R_{+}}$, which is closely related to $L_2$ regularization and weight decay. Note that $\vtheta_\text{MAP}$ corresponds to a mode of the posterior distribution $p(\vtheta | \mathcal{D})$ and thus it is not unique since $p(\vtheta | \mathcal{D})$ is generally multi-modal---different training procedures, or even different random initializations of $\vtheta$, could yield different solutions to \eqref{eq:map_estimation}.

Exploiting the randomness within MAP estimation, Deep Ensembles \citep[DE,][]{Lakshminarayanan2016SimpleAS} aggregate the predictions of multiple MAP estimates arising from different random parameter initializations. More formally, given a sequence of $K$ distinct MAP estimates $(\vtheta_\text{MAP}^{(k)})_{k=1}^K$, a DE simply averages (i.e.~with uniform weights) the predictions of the induced networks to obtain the predictive distribution
\begin{equation} \label{eq:de_pred}
    p_\text{DE}(\vy | \vx, \mathcal{D}) := \frac{1}{K} \sum_{k=1}^K p(\vy | \vx, \vtheta_\text{MAP}^{(k)}) .
\end{equation}
The number of ensemble members $K$ is typically chosen to be small, e.g. $K=5$. Despite their simplicity, DE has been shown to yield state-of-the-art results in uncertainty quantification \citep{Ovadia2019CanYT}.

The MAP estimate $\vtheta_\text{MAP}$ represents a single point estimate in the parameter space and hence it ignores the uncertainty inherent in the parameters of the network. A Bayesian treatment of $\rvf_\vtheta$, which results in a Bayesian neural network (BNN), attempts to capture this uncertainty by inferring the full posterior distribution $p(\vtheta | \mathcal{D})$. Alas, this is computationally intractable due to the nonlinear nature of the network $\rvf_\vtheta$, requiring the use of approximate inference techniques. Given an approximate, easy-to-sample-from posterior $q(\vtheta | \mathcal{D})$, predictions can then be made via Monte Carlo (MC) integration:
\begin{equation} \label{eq:mc_integral}
    p_\text{BNN}(\vy | \vx, \mathcal{D}) = \int p(\vy | \vx, \vtheta) \, q(\vtheta | \mathcal{D}) \,d\vtheta
        \ \approx \ \frac{1}{S} \sum_{s=1}^S p(\vy | \vx, \vtheta^{(s)}) ; \quad \vtheta^{(s)} \sim q(\vtheta | \mathcal{D}) .
\end{equation}

\subsection{(Last-Layer) Laplace Approximation}
\label{subsec:ll}

A standard way to define the approximate posterior $q(\vtheta | \mathcal{D})$ is via a simple Gaussian approximation $\mathcal{N}(\vtheta|\vmu, \mSigma)$. One option is to employ the Laplace approximation \cite{mackay1992practical}.
In this technique, we apply a second-order Taylor expansion at $\vtheta_{\text{MAP}}$ to the log joint distribution, and then choose $\vmu = \vtheta_{\text{MAP}}$ and $\mSigma = (\mH_\vtheta + \lambda \mathbf{I})^{-1} \in \mathbb{R}^{M \times M},$ where $\mH_\vtheta$ is the Hessian of $-\log p(\mathcal{D}|\vtheta)$ w.r.t. $\vtheta$ evaluated at $\vtheta_{\text{MAP}}$.

Recently it has been shown that last-layer Bayesian approximations yield a competitive performance compared to all-layer ones \citep{SnoekRSKSSPPA15,ober2019benchmarking,brosse2020last}. In this case, we consider the network $\rvf_\vtheta$ as a linear function in the weight matrix $\mW \in \mathbb{R}^{C \times P}$ of the last layer, i.e. $\rvf_{\mW}(\vx_*) = \mW \vphi(\vx_*),$ with \emph{fixed} features $\vphi(\vx_*) =: \vphi_* \in \mathbb{R}^P$, which are simply the output of the penultimate layer of the neural network given an input $\vx_*$. Note, that this formulation includes the case where we have a bias parameter, by simply employing the standard bias trick. In this setting, we therefore only need to do a Laplace approximation on a single weight matrix $\mW$: we obtain a Gaussian approximation $\mathcal{N}(\mathrm{vec}(\mW) | \vmu, \mSigma)$ with $\vmu = \text{vec}(\mW_{\text{MAP}}) \in \mathbb{R}^{CP}$ and $\mSigma = (\mH_\mW + \lambda \mathbf{I})^{-1} \in \mathbb{R}^{CP \times CP},$ where $\mH_\mW$ is the Hessian of $-\log p(\mathcal{D}|\mW)$ w.r.t. $\text{vec}(\mW)$ evaluated at $\text{vec}(\mW_{\text{MAP}})$. This method is called the \emph{last-layer Laplace approximation} \citep[LLLA,][]{Kristiadi2020BeingBE}.

Let $\vx_* \in \R^D$ be an arbitrary test point. Since $\rvf_{\mW}$ is linear in $\mW,$ we also have a Gaussian distribution $p(\rvf_* | \vx_*, \mathcal{D}) = \mathcal{N}(\rvf_* |\vm_*, \mC_*)$ over the marginal network outputs $\rvf_* := \rvf(\vx_*)$, where $\vm_* = \mW_{\text{MAP}} \vphi_* \in \mathbb{R}^C$ and $\mC_* = (\vphi_*^\top \otimes \mathbf{I}) \mSigma (\vphi_* \otimes \mathbf{I}) \in \mathbb{R}^{C \times C}$.\footnote{$\otimes$ denotes the Kronecker product.} For multi-class classification, the predictive distribution is therefore given by
\begin{equation} \label{eq:pred}
    p(\vy=\ve_c | \vx_*, \mathcal{D}) = \int \sigma(\rvf_*)_c \, p(\rvf_*|\mathcal{D}) \, d\rvf_*,
\end{equation}
where $\sigma(\rvf_*)_c = \exp(\rvf_{*c}) / \sum_{i=1}^C \exp(\rvf_{*i})$ is the softmax function. This integral does not have an analytic solution, but it can be approximated via the \emph{multi-class probit approximation} \citep[MPA,][]{Gibbs98, lu2021meanfield}:
\begin{equation} \label{eq:extended_mackay}
        p(\vy=\ve_c|\vx_*, \mathcal{D}) \approx \sigma(\vz_*)_c,
\end{equation}
where $\vz_*$ is a vector with $i$-th component $\vz_{*i} = \vm_{*i} / \sqrt{1 + (\pi/8) \mC_{*ii}}.$ Intuitively, the output mean $\vm_{*i}$ is scaled by a factor which depends on the (non-negative) output variance $\mC_{*ii}$, thus $\vz_{*i}$ can only stay the same or decrease. This is conceptually similar to temperature scaling \citep{guo2017calibration}, but with a different temperature parameter $T \geq 1$ for each class and data point.

\section{Mixtures of Laplace Approximations}

Comparing \eqref{eq:de_pred} and \eqref{eq:mc_integral}, one can notice that DE approximate the BNN predictive distribution. Why, then, do BNNs perform worse than DE? One hypothesis is that each MAP estimate within DE constitutes a mode of $p(\vtheta | \mathcal{D})$---therefore, DE can capture $K$ modes of the posterior, whereas Gaussian-based BNNs can only capture one mode. DE can thus be seen as capturing \emph{global} uncertainty of the posterior, whereas Gaussian-based BNNs capture \emph{local} uncertainty \citep{Fort2019DeepEA, Wilson2020BayesianDL}.
In this section, we build on top of the MAP estimates provided by a DE to construct a BNN with a MoG posterior in a \emph{post-hoc} manner via Laplace approximations. We call the resulting method \emph{Mixtures of Laplace Approximations} (MoLA). While MoLA can be applied using any type of Laplace approximation, in this work we focus on a lightweight variant, where LLLA is employed. We henceforth assume LLLA by default when referring to MoLA.

Consider DE's MAP estimates $(\vtheta_\text{MAP}^{(k)})_{k=1}^K$ of an $L$-layer network. For each $k = 1 \dots, K$, we treat the first $L-1$ layers as a fixed feature map $\vphi^{(k)}$ and construct a Gaussian approximate posterior $\mathcal{N}(\mathrm{vec}(\mW) | \vmu^{(k)}, \mSigma^{(k)})$ over the last-layer weights via LLLA. That is, we set $\vmu^{(k)} = \mathrm{vec}(\mW^{(k)}_\text{MAP})$ and $\mSigma^{(k)}$ to be the inverse Hessian of the negative log-posterior w.r.t. $\mathrm{vec}(\mW)$ at $\vmu^{(k)}$. Given a sequence $\boldsymbol{\pi} := (\pi^{(k)})_{k=1}^K$ of non-negative real numbers with $\sum_{k=1}^K \pi^{(k)} = 1$, we define the approximate posterior as\footnote{The choice of $\boldsymbol{\pi}$ shall be discussed in \Cref{subsec:practical}.}
\begin{equation} \label{eq:mola_posterior}
    p_\text{MoLA}(\mathrm{vec}(\mW) | \mathcal{D}) := \sum_{k=1}^K \pi^{(k)} \, \mathcal{N}(\mathrm{vec}(\mW) | \vmu^{(k)}, \mSigma^{(k)}) .
\end{equation}
This posterior can be seen intuitively as endowing each of DE's modes with a sense of local uncertainty, cf. \Cref{fig:one} for an illustration.
MoLA's posterior also induces a MoG distribution over the network outputs due to the linearity of $\rvf_{\mW}$ in $\mW$. That is, given any input $\vx_* \in \R^D$, we have
\begin{equation} \label{eq:mola_output_dist}
    p_\text{MoLA}(\rvf_* | \vx_*, \mathcal{D}) := \sum_{k=1}^K \pi^{(k)} \, \mathcal{N}(\rvf_* | \vm_*^{(k)}, \mC_*^{(k)}).
\end{equation}
As a consequence, by employing the approximation from \eqref{eq:extended_mackay}, MoLA's predictive distribution takes a particularly simple form:
\begin{equation} \label{eq:mola_predictive}
    p_\text{MoLA}(\vy = \ve_c | \vx_*, \mathcal{D}) \approx \sum_{k=1}^K \pi^{(k)} \, \sigma(\vz_*^{(k)})_c ,
    \end{equation}
where $\vz_*^{(k)}$ is as defined in \eqref{eq:extended_mackay} for each $k = 1, \dots, K$ (see \Cref{sec:proofs} for the derivation).
We note that \eqref{eq:mola_predictive} is simply the weighted sum of the predictive distributions \eqref{eq:extended_mackay} induced by MoLA's mixture components. Due to the LLLA used, this can be computed at low overhead---cheaper than standard MC-integration. While other kinds of Laplace approximations can also satisfy this, they require a costly linearization over the entire network's parameters. LLLA is thus a favorable practical choice.

\begin{table}[t]
    \caption{The memory and computation costs of all methods after training in $\mathcal{O}$ notation. We also measure the wall-clock time (in seconds) and memory (in megabyte) for a Wide-ResNet on the CIFAR-10 test set, on a single NVIDIA RTX 2080Ti GPU. $N, M, K, C, P$ are defined in \Secref{sec:background}. For SWAG and MultiSWAG (MSWAG), $S$ denotes the number of MC samples used to approximate the predictive distribution, and $R$ denotes the number of model snapshots used to approximate the posterior. For prediction, the theoretical complexity is for a single test point.}
    \label{tab:costs}
    \vspace{0.5em}

    \centering
    \fontsize{8}{9}\selectfont
    \renewcommand{\tabcolsep}{4pt}
    \begin{sc}
    \begin{tabular}{lccrccccrcr}
        \toprule
            & & \multicolumn{2}{c}{Inference} & & & & \multicolumn{4}{c}{Prediction} \\
            Method & & \multicolumn{2}{c}{Computation} & & & & \multicolumn{2}{c}{Computation} &  \multicolumn{2}{c}{Memory} \\
            \midrule
            MAP &        & - & {-} & & & & $M$ & [1.17s] & $M$ & [11MB] \\
            LLLA & & $NM \hspace{-0.5mm}+ \hspace{-0.5mm}C^3 \hspace{-0.5mm}+ \hspace{-0.5mm}P^3$ & [42.58s] & & & & $M$ & [1.14s] & $M \hspace{-0.5mm}+ \hspace{-0.5mm}C^2 \hspace{-0.5mm}+ \hspace{-0.5mm}P^2$ & [12MB] \\
            SWAG & & $RNM$ & [1310.81s] & & & & $SRM$ & [21.90s] & $RM$ & [440MB] \\
            \midrule
            DE &   & - & {-}  & & & & $\textcolor{red}{K}M$ & [3.80s] & $\textcolor{red}{K}M$ & [55MB] \\
            MoLA &  & $\textcolor{red}{K}(NM \hspace{-0.5mm} + \hspace{-0.5mm}C^3 \hspace{-0.5mm}+ \hspace{-0.5mm}P^3)$ & [155.70s] & & & & $\textcolor{red}{K}M$ & [4.04s] & $\textcolor{red}{K}(M \hspace{-0.5mm}+\hspace{-0.5mm} C^2\hspace{-0.5mm} +\hspace{-0.5mm} P^2)$ & [58MB] \\
            MSWAG &  & $\textcolor{red}{K}RNM$ & [6718.42s] & & & & $\textcolor{red}{K}SRM$ & [114.48s] & $\textcolor{red}{K}RM$ & [2200MB] \\
        \bottomrule
    \end{tabular}
    \end{sc}
\end{table}

\subsection{Analysis}
\label{subsec:analysis}

Feed-forward DNNs with ReLU nonlinearity---the so-called \emph{ReLU networks}---are provably overconfident ``far away'' from the training data, in the sense that a training input $\vx$ is scaled with a scalar $\delta > 0$, and as $\delta \rightarrow \infty$ \cite{Hein_2019_CVPR}. For the binary classification setting it has been shown that an approximate (Gaussian-form) Bayesian treatment can mitigate this issue \citep{Kristiadi2020BeingBE}. We are thus interested to know whether this desirable property of single-Gaussian BNNs also holds for MoLA.\footnote{The results in this section also hold for general last-layer Gaussian- (\Cref{lemma:lll_multiclass}) and MoG-based (\Cref{thm:mola_confidence}) BNNs. We focus on MoLA for clarity.} Our analysis omits the bias parameters of the network and we employ the multi-class probit approximation \eqref{eq:extended_mackay} for analytical tractability. All proofs are in \Cref{sec:proofs}.

As a preliminary, we define the \emph{confidence} of a prediction for an input $\vx_* \in \R^D$ by $\max_{i \in \{1, ..., C\}} p(\vy = \ve_i|\vx_*, \mathcal{D})$, i.e. it is the probability associated with the predicted class label.
Also, note that the multi-class probit approximation only uses the diagonal elements $\mC_{*ii}$ of the output covariance matrix $\mC_*$. Hence, without loss of generality, for our analysis we only need to consider the posterior distributions $\mathcal{N}(\vw_{(i)} | \vmu_{(i)}, \mSigma_{(i)})$ over each row $\vw_{(i)}$ of the weight matrix $\mW$, instead of the full posterior over $\mathrm{vec}(\mW)$.

First, we present the following result which holds for single-Gaussian BNNs, as an extension of \citet{Kristiadi2020BeingBE}'s analysis to the multi-class classification case. It shows that asymptotically, the confidence of any input point is bounded away from the maximum confidence of $1$ and the tightness of this bound depends on the uncertainty encoded in the covariance matrix of the approximate posterior.

\vspace{0.5em}
\begin{restatable}{lemma}{lllmulticlass} \label{lemma:lll_multiclass}
    Suppose that $\rvf_\mW: \R^D \to \R^C$ is a ReLU network without bias. For any non-zero $\vx_* \in \mathbb{R}^D,$ there exists $\alpha > 0$ such that for all $\delta \geq \alpha$, we have under the multi-class probit approximation \eqref{eq:extended_mackay}
    \begin{equation*}
        p_\text{\emph{BNN}}(\vy = \ve_{c_*} | \delta \vx_*, \mathcal{D}) \leq \frac{1}{1 + \sum_{i \neq c_*} \exp(-(b_i + b_{c_*}))},
    \end{equation*}
    where $c_* = \argmax_{i \in \{1, ..., C\}} p_\text{\emph{BNN}}(\vy = \ve_i | \delta\vx_*, \mathcal{D})$ is the predicted class label and for each $i = 1, \dots, C$, and we define $b_i := \Vert \vmu_{(i)} \Vert_2 / \sqrt{(\pi/8) \lambda_{\emph{min}}(\mSigma_{(i)})},$ with $\lambda_{\emph{min}}(\mSigma_{(i)})$ being the smallest eigenvalue of $\mSigma_{(i)}$.
\end{restatable}

To get an intuition for the behavior of the bound, we can consider two cases: When $\lambda_{\text{min}}(\mSigma_{(i)}) \rightarrow 0$ for all $i = 1, \dots, C$ the upper bound on the confidence approaches one, i.e. minimal uncertainty. Conversely, if $\forall i \in \{1, ..., C\}: \lambda_{\text{min}}(\mSigma_{(i)}) \rightarrow \infty,$ the upper bound approaches $1 / C,$ i.e. maximal uncertainty. This confirms the intuition that increased uncertainty in the weight space results in increased uncertainty in the function space.

Now we can easily use the upper bound provided by \Cref{lemma:lll_multiclass} to obtain an asymptotic confidence bound for MoLA. As in \Cref{lemma:lll_multiclass}, we only need to consider the posterior distributions $p_\text{MoLA}(\vw_{(i)} | \mathcal{D}) = \sum_{k=1}^K \pi^{(k)} \mathcal{N}(\vw_{(i)} | \vmu_{(i)}^{(k)}, \mSigma_{(i)}^{(k)})$ over each row $\vw_{(i)}$ of the weight matrix $\mW$, instead of the full MoLA posterior in \eqref{eq:mola_posterior}.
\vspace{0.5em}
\begin{restatable}{theorem}{molaconfidence} \label{thm:mola_confidence}
    Suppose that $\rvf_\mW: \R^D \to \R^C$ is a ReLU network without bias, equipped with a MoLA posterior $p_\text{\emph{MoLA}}(\vw_{(i)} | \mathcal{D})$ for each $i = 1, \dots, C.$ Using the approximation in \eqref{eq:mola_predictive}, for any non-zero $\vx_* \in \mathbb{R}^D$ there exists $\alpha > 0$ such that for all $\delta \geq \alpha,$ we have
    \begin{align*}
        p_\text{\emph{MoLA}}(\vy = \ve_{c_*} | \delta \vx_*, \mathcal{D})
           \ \leq \ \sum_{k=1}^K  \frac{\pi^{(k)}}{1 + \sum_{i \neq c_*} \exp(-(b_i^{(k)} + b_{c_*}^{(k)}))},
    \end{align*}
    where for each $k = 1, \dots, K$, the integer $c_*^{(k)}$ is the predicted class label according to \eqref{eq:extended_mackay} and $b_i^{(k)}$ is defined as in \Cref{lemma:lll_multiclass}, both under the $k$-th mixture component of $p_\text{\emph{MoLA}}.$
\end{restatable}
Note that when all mixture components of $p_\text{MoLA}$ are identical and $\boldsymbol{\pi}$ is a uniform probability vector, then we recover the result in \Cref{lemma:lll_multiclass}. However, this is unlikely under the assumption that each MAP estimate on which MoLA is built upon is obtained via a random initialization. In fact, MoLA's bound can be tighter than that of \Cref{lemma:lll_multiclass} since intuitively it seems unlikely that the bounds on all components are worse than the one of any single randomly chosen component, given that they are all obtained via random initialization. 

\subsection{Practical Considerations}
\label{subsec:practical}

\paragraph*{Alternative Laplace approximation for MoLA.} While Laplace approximations over all weights might not be feasible for very large networks, one can also use lightweight variants of Laplace approximations, such as a Laplace approximation over a \emph{subset} of the weights \citep{daxberger2020bayesian}, instead of LLLA.

\paragraph*{Kronecker-factored approximation of the Hessian.} If the output dimension of the DNN, i.e. the number of classes $C$, is sufficiently small, one can usually use the full Hessian.
However, for problems with many classes, like ImageNet ($C = 1000$), this becomes infeasible for even the LLLA. To ensure general applicability, we use a Kronecker factored (K-FAC) generalized Gauss-Newton (GGN) approximation \cite{martens2015optimizing, Ritter2018ASL} which can be efficiently computed using automatic differentiation \cite{dangel2020backpack}. 
Note that for LLLA the GGN and the Hessian coincide due to the linearity of the output layer. Under the K-FAC approximation, we only have to compute, invert, and store two matrices of size $C \times C$ and $P \times P$, instead of one $CP \times CP$ matrix. Moreover, we empirically find that the difference in performance between a K-FAC and full GGN is relatively small (cf. \Cref{fig:llla} for example); hence, we choose the K-FAC approximation in all our experiments.

\paragraph*{Choice of mixture weights.} So far we have yet to discuss the choice of the mixture weights $\boldsymbol{\pi}$ of the MoLA posterior. 
Intuitively, one way to  pick a mixture coefficient is by picking it proportional to the importance of the corresponding mixture component. The marginal likelihood of each the mixture component can be used to do so \citep{mackay1992practical, immer2021scalable}, leading to the following mixture coefficient
\begin{equation}
    \pi^{(k)} := \frac{p(\mathcal{D}|\mathcal{M}_k)}{\sum_{k^{'}=1}^K p(\mathcal{D}|\mathcal{M}_{k^{'}})} \qquad \text{for all}\enspace k = 1, \dots, K ,
\end{equation}
where $\mathcal{M}_k$ is the model (i.e. the choice of architecture and hyperparameters) of the $k$th mixture component. See \Cref{sec:marglik_weights} for a more formal derivation. In our setting, where the only difference between the mixture components is the random initialization of the DNN weights, all models $\mathcal{M}_k$ are identical. Hence, this approach should lead to uniform mixture weights $\pi^{(k)} = 1/K$ for all $k = 1, \dots, K$. We confirm this empirically and find that the marginal likelihoods of the components only vary minimally, consistent with what \citet{immer2021scalable} report. Therefore, in all our experiments, we use this uniform weighting. 

\paragraph*{\textit{Post-hoc} tuning of prior precision.} To be able to apply the Laplace approximation \textit{post-hoc} on an arbitrary pre-trained DNN, the prior precision $\lambda$ can usually not be set to the value used for $L_2$ regularization or weight decay during training \cite{Ritter2018ASL, immer2020improving}.
Thus, we tune the prior precision for MoLA, assuming a single prior precision for all mixture components---consistent with what \citet{rahaman2020uncertainty} report for temperature scaling on a regular ensemble---of $p_\text{MoLA}(\mathrm{vec}(\mW) | \mathcal{D}).$ For our experiments we use thresholds on the validation set's average confidence and Brier score, choosing the smallest prior precision which results in meeting the thresholds (cf. \Cref{sec:details}).
Alternatively, the prior precision can also be tuned by optimizing the components' marginal likelihoods \cite{immer2021scalable}, w.r.t. to a proper scoring rule like the validation Brier score or log-likelihood, or, using OOD data \cite{hendrycks2019oe}.

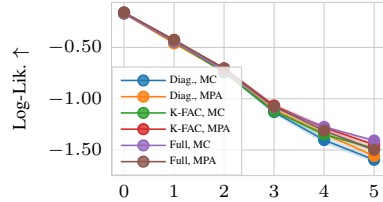
\begin{wrapfigure}[11]{r}{0.375\textwidth}
    \centering
    \vspace{-1em}

    \begin{tikzpicture}

\definecolor{color0}{rgb}{0.12156862745098,0.466666666666667,0.705882352941177}
\definecolor{color1}{rgb}{1,0.498039215686275,0.0549019607843137}
\definecolor{color2}{rgb}{0.172549019607843,0.627450980392157,0.172549019607843}
\definecolor{color3}{rgb}{0.83921568627451,0.152941176470588,0.156862745098039}
\definecolor{color4}{rgb}{0.580392156862745,0.403921568627451,0.741176470588235}
\definecolor{color5}{rgb}{0.549019607843137,0.337254901960784,0.294117647058824}

\tikzstyle{every node}=[font=\scriptsize]

\begin{groupplot}[group style={group size=2 by 2}]
\nextgroupplot[
width=\linewidth,
height=0.17\textheight,
legend cell align={left},
legend style={nodes={scale=0.6, transform shape},fill opacity=0.8, draw opacity=1, text opacity=1, at={(0, 0)}, anchor=south west, draw=white!80!black},
axis line style={white!80!black},
tick align=inside,
x grid style={white!80!black},
xmajorgrids,
xmin=-0.25, xmax=5.25,
xtick style={draw=none},
xtick={-1,0,1,2,3,4,5,6},
y grid style={white!80!black},
ylabel={Log-Lik. $\uparrow$},
ymajorgrids,
ymin=-1.75, ymax=-0.06,
ytick style={draw=none},
ytick={-1.5,-1,-0.5,0},
yticklabels={\(\displaystyle  {−1.50}\),\(\displaystyle  {−1.00}\),\(\displaystyle  {−0.50}\),\(\displaystyle {0.00}\)}
]
\path [draw=white, fill=color0, opacity=0.2]
(axis cs:0,-0.165892607769898)
--(axis cs:0,-0.161680069270202)
--(axis cs:1,-0.434722075102017)
--(axis cs:2,-0.718617472028238)
--(axis cs:3,-1.1015382796368)
--(axis cs:4,-1.36482149754111)
--(axis cs:5,-1.56387169006552)
--(axis cs:5,-1.62834011731261)
--(axis cs:5,-1.62834011731261)
--(axis cs:4,-1.44205710430876)
--(axis cs:3,-1.15873750282118)
--(axis cs:2,-0.755788163118856)
--(axis cs:1,-0.447993227988396)
--(axis cs:0,-0.165892607769898)
--cycle;

\path [draw=white, fill=color1, opacity=0.2]
(axis cs:0,-0.170980025752919)
--(axis cs:0,-0.167937613597972)
--(axis cs:1,-0.454485860598909)
--(axis cs:2,-0.722770689725162)
--(axis cs:3,-1.0841502560853)
--(axis cs:4,-1.31066248200441)
--(axis cs:5,-1.53111397120391)
--(axis cs:5,-1.58452025724814)
--(axis cs:5,-1.58452025724814)
--(axis cs:4,-1.37369472421304)
--(axis cs:3,-1.12889708922873)
--(axis cs:2,-0.75461795095833)
--(axis cs:1,-0.46475402607915)
--(axis cs:0,-0.170980025752919)
--cycle;

\path [draw=white, fill=color2, opacity=0.2]
(axis cs:0,-0.168538434570739)
--(axis cs:0,-0.163318285590699)
--(axis cs:1,-0.433402027717032)
--(axis cs:2,-0.717035981338182)
--(axis cs:3,-1.08737655186496)
--(axis cs:4,-1.30622338662305)
--(axis cs:5,-1.45347246198313)
--(axis cs:5,-1.51701262464864)
--(axis cs:5,-1.51701262464864)
--(axis cs:4,-1.39022261643888)
--(axis cs:3,-1.15492119136014)
--(axis cs:2,-0.765620970502537)
--(axis cs:1,-0.454137342654468)
--(axis cs:0,-0.168538434570739)
--cycle;

\path [draw=white, fill=color3, opacity=0.2]
(axis cs:0,-0.161031724204362)
--(axis cs:0,-0.157540758858382)
--(axis cs:1,-0.420121332643638)
--(axis cs:2,-0.686263250317282)
--(axis cs:3,-1.04002498037989)
--(axis cs:4,-1.25232322460452)
--(axis cs:5,-1.42243196265747)
--(axis cs:5,-1.47427104193797)
--(axis cs:5,-1.47427104193797)
--(axis cs:4,-1.31306269627612)
--(axis cs:3,-1.09274448923732)
--(axis cs:2,-0.722579113504384)
--(axis cs:1,-0.434277249007413)
--(axis cs:0,-0.161031724204362)
--cycle;

\path [draw=white, fill=color4, opacity=0.2]
(axis cs:0,-0.17144105052596)
--(axis cs:0,-0.166741693881694)
--(axis cs:1,-0.435798122149438)
--(axis cs:2,-0.709088990089258)
--(axis cs:3,-1.06186540213448)
--(axis cs:4,-1.25466363860893)
--(axis cs:5,-1.39366498200185)
--(axis cs:5,-1.42301948853088)
--(axis cs:5,-1.42301948853088)
--(axis cs:4,-1.29571446432941)
--(axis cs:3,-1.10254564004081)
--(axis cs:2,-0.738592575297197)
--(axis cs:1,-0.447508251313239)
--(axis cs:0,-0.17144105052596)
--cycle;

\path [draw=white, fill=color5, opacity=0.2]
(axis cs:0,-0.160215811519028)
--(axis cs:0,-0.156657397099136)
--(axis cs:1,-0.419035112880339)
--(axis cs:2,-0.686920769432478)
--(axis cs:3,-1.05025526855592)
--(axis cs:4,-1.27896164843756)
--(axis cs:5,-1.47114243959528)
--(axis cs:5,-1.52805869587479)
--(axis cs:5,-1.52805869587479)
--(axis cs:4,-1.34428181556187)
--(axis cs:3,-1.10221853793975)
--(axis cs:2,-0.721318185073443)
--(axis cs:1,-0.431484198096325)
--(axis cs:0,-0.160215811519028)
--cycle;

\addplot [semithick, color0, mark=*, mark size=2, mark options={solid}]
table {%
0 -0.16378633852005
1 -0.441357651545207
2 -0.737202817573547
3 -1.13013789122899
4 -1.40343930092494
5 -1.59610590368907
};
\addlegendentry{Diag., MC}

\addplot [semithick, color1, mark=*, mark size=2, mark options={solid}]
table {%
0 -0.169458819675446
1 -0.45961994333903
2 -0.738694320341746
3 -1.10652367265701
4 -1.34217860310872
5 -1.55781711422602
};
\addlegendentry{Diag., MPA}

\addplot [semithick, color2, mark=*, mark size=2, mark options={solid}]
table {%
0 -0.165928360080719
1 -0.44376968518575
2 -0.741328475920359
3 -1.12114887161255
4 -1.34822300153097
5 -1.48524254331589
};
\addlegendentry{K-FAC, MC}

\addplot [semithick, color3, mark=*, mark size=2, mark options={solid}]
table {%
0 -0.159286241531372
1 -0.427199290825526
2 -0.704421181910833
3 -1.0663847348086
4 -1.28269296044032
5 -1.44835150229772
};
\addlegendentry{K-FAC, MPA}

\addplot [semithick, color4, mark=*, mark size=2, mark options={solid}]
table {%
0 -0.169091372203827
1 -0.441653186731338
2 -0.723840782693227
3 -1.08220552108765
4 -1.27518905146917
5 -1.40834223526637
};
\addlegendentry{Full, MC}

\addplot [semithick, color5, mark=*, mark size=2, mark options={solid}]
table {%
0 -0.158436604309082
1 -0.425259655488332
2 -0.70411947725296
3 -1.07623690324783
4 -1.31162173199972
5 -1.49960056773504
};
\addlegendentry{Full, MPA}

\end{groupplot}

\end{tikzpicture}

    \caption{Comparisons of LLLAs on CIFAR-10-C \citep{hendrycks2018benchmarking} in terms of log-likelihood.}
    \label{fig:llla}
\end{wrapfigure}

\paragraph*{Approximating the predictive distribution.} While one can use the standard MC-integral \eqref{eq:mc_integral} to approximate the predictive distribution $p_\text{{MoLA}}(\vy = \ve_c | \vx_*, \mathcal{D})$, in our last-layer setting the closed-form MPA \eqref{eq:extended_mackay} is not only theoretically useful, but also computationally efficient (cf. \Cref{tab:costs}). Moreover, it achieves competitive uncertainty calibration in terms of log-likelihood (\Cref{fig:llla}).

\paragraph*{Efficient ensembling techniques.} When no set of pre-trained DNNs is available, MoLA is more expensive than single-model methods. In any case, MoLA requires $K$ forward passes for prediction. Leveraging recent work by \citet{Havasi2021mimo}, we can leverage their multi-input multi-output (MIMO) method to implement MoLA within a \emph{single} DNN. The resulting method MIMO-MoLA applies MoLA to multiple independent subnetworks within one DNN. We tested MIMO-MoLA on CIFAR-10-C as a proof-of-concept and find very competitive performance compared to the other methods, especially for a single-model method (see \Cref{fig:mnist_cifar10}).

\paragraph*{Limitations.}

As opposed to other common Bayesian deep learning methods, MoLA cannot be directly applied to other interesting problem domains besides predictive uncertainty quantification, such as continual learning -- we leave this for future work.

\section{Related Work}
\label{sec:related_work}

Bayesian deep learning has a long history and many methods besides the Laplace approximation have been proposed to infer an approximate posterior, such as VI \cite{Graves2011PracticalVI, blundell2015weight, Khan2018FastAS, Zhang2018NoisyNG} and Markov Chain Monte Carlo \cite{neal1996, Welling2011BayesianLV, zhang2019cyclical} methods.
The idea of capturing multiple posterior modes with a BNN has previously been explored. \citet{Wilson2020BayesianDL} proposed a method called MultiSWAG which takes multiple samples around several MAP estimates of a DNN to construct a MoG approximation.
Other methods that also sample around a single mode have been combined with Deep Ensembles, such as subspace sampling \citep[][similar to SWAG]{Fort2019DeepEA} and MC dropout \citep{Filos2019ASC}.
Meanwhile, \citet{dusenberry2020efficient} proposed to construct a MoG posterior via VI with an efficient rank-one parameterization. Their method requires full training (thus non \emph{post-hoc}) but is more efficient compared to training multiple DNNs from scratch.

The efficacy of last-layer Bayesian approximations have been shown previously \citep{SnoekRSKSSPPA15, riquelme2018deep, Pinsler0NH19, Ovadia2019CanYT, Liu2020SimpleAP, ober2019benchmarking, brosse2020last}. Especially for the Gaussian-based variants, their properties have theoretically been studied by \citet{Kristiadi2020BeingBE}. However \emph{mixtures} of last-layer approximate posteriors have not been studied---both theoretically and empirically---so far.

Recently, the linearization of DNNs in the context of Laplace approximations has been explored \cite{khan2019approximate, Foong2019InBetweenUI, immer2020improving, Kristiadi2020BeingBE}. While this can also lead to different approximations to the predictive distribution, it requires the computation of Jacobian matrices which is generally expensive. In contrast, MoLA can make use of the multi-class probit approximation at a low overhead due to its last-layer nature.

\section{Extended Experiments}
\label{sec:extended_experiments}

\subsection{Setup}

We focus on evaluating the predictive uncertainty of MoLA on image data. Specifically, we consider the benchmark datasets of \citet{hendrycks2018benchmarking, Ovadia2019CanYT}: (i) the rotated MNIST (MNIST-R) dataset, which consists of transformed MNIST images, rotated with increasing angle up to 180 degrees, (ii) the corrupted CIFAR-10 (CIFAR-10-C), and (iii) corrupted ImageNet (ImageNet-C) datasets, both consisting of 5 severity levels of 16/19 different perturbations of the respective original dataset (cf. \Cref{sec:details}).
Moreover, we also study the performance of MoLA in OOD detection tasks with MNIST and CIFAR-10 as the in-distribution datasets.
\begin{figure*}[t]
    \centering

    \input{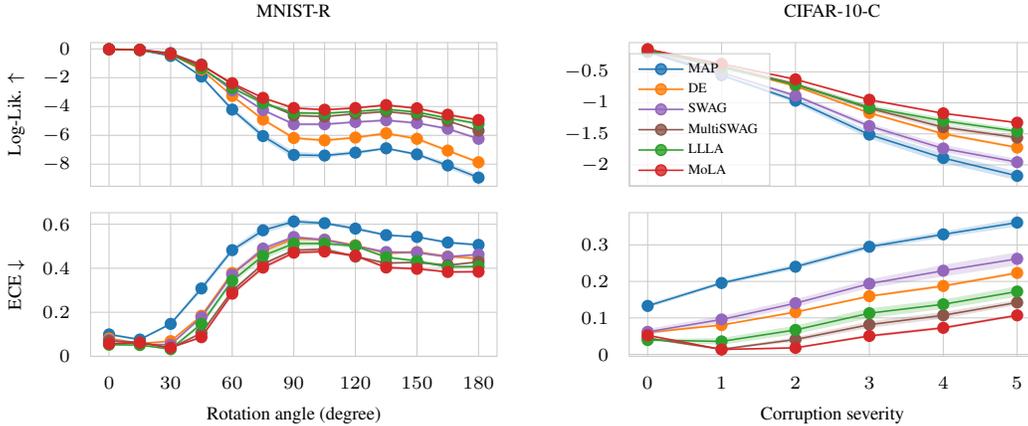}

    \vspace{-1em}
    \caption{Comparison of all methods on MNIST-R (\textbf{left column}) and corrupted CIFAR-10-C (\textbf{right column}), in terms of the log-likelihood (\textbf{top row}, higher is better) and expected calibration error (\textbf{bottom row}, lower is better) metrics. Lines and dots represent means while shades represent standard error of five runs with models trained with different random initalizations.}
    \label{fig:mnist_cifar10}
\end{figure*}

We compare MoLA against the following baselines: (i) the standard MAP-trained network (MAP), (ii) Deep Ensemble \citep[DE,][]{Lakshminarayanan2016SimpleAS}, (iii) Stochastic Weight Averaging Gaussian \citep[SWAG,][]{maddox2019simple}, (iv) MultiSWAG \citep[MSWAG,][]{Wilson2020BayesianDL}, (v) and the single-Gaussian-based LLLA.

Note in particular that MultiSWAG represents the state-of-the-art \emph{post-hoc} MoG-based Bayesian methods. While the method of \citet{dusenberry2020efficient} also constructs MoG posteriors, its performance is similar to Deep Ensemble. Moreover, Deep Ensemble has been shown to perform better than other methods like VI, MC-Dropout, and temperature scaling---at least on larger benchmarks like CIFAR-10-C and especially ImageNet-C \citep{Ovadia2019CanYT}. Thus, the Deep Ensemble baseline is already representative of these methods. We do not compare against SWAG and MultiSWAG on ImageNet-C, due to their high computational and memory cost, see \Cref{subsec:cost} for context. We use standard DNN architectures (LeNet, ResNet, and WideResNet) and training methods. For the MNIST and CIFAR-10 experiments, we use five mixture components, while for ImageNet we use three. Full detail in \Cref{sec:details}.

To evaluate the distribution shift tasks, we use the log-likelihood (LL) and expected calibration error \citep[ECE,][]{naeini2015obtaining} metrics.
We also provide results for accuracy, Brier score \cite{brier1950verification}, mean confidence (MMC), and maximum calibration error \citep[MCE,][]{naeini2015obtaining} in \Cref{sec:full_results}. For OOD detection, on top of MMC, we use the area under the ROC curve (AUROC) metric \citep{hendrycks17baseline}.

\begin{figure*}[t]
    \centering

    \begin{tikzpicture}

\definecolor{color0}{rgb}{0.12156862745098,0.466666666666667,0.705882352941177}
\definecolor{color1}{rgb}{1,0.498039215686275,0.0549019607843137}
\definecolor{color2}{rgb}{0.172549019607843,0.627450980392157,0.172549019607843}
\definecolor{color3}{rgb}{0.83921568627451,0.152941176470588,0.156862745098039}

\tikzstyle{every node}=[font=\scriptsize]

\begin{groupplot}[group style={group size=2 by 1, horizontal sep=5em}]
\nextgroupplot[
width=0.5\textwidth,
height=0.25\textwidth,
axis line style={white!80!black},
tick align=outside,
x grid style={white!80!black},
xlabel={Corruption severity},
xmajorgrids,
xmin=-0.25, xmax=5.25,
xtick style={draw=none},
xtick={-1,0,1,2,3,4,5,6},
xticklabels={\(\displaystyle {−1}\),\(\displaystyle {0}\),\(\displaystyle {1}\),\(\displaystyle {2}\),\(\displaystyle {3}\),\(\displaystyle {4}\),\(\displaystyle {5}\),\(\displaystyle {6}\)},
y grid style={white!80!black},
ylabel={Log-Lik. $\uparrow$},
ymajorgrids,
ymin=-5.03516404281968, ymax=-0.598977141545446,
ytick style={draw=none},
]
\addplot [semithick, color0, mark=*, mark size=2, mark options={solid}]
table {%
0 -0.917663623491923
1 -1.61928975715236
2 -2.19399937905161
3 -2.82555094541048
4 -3.73790527582269
5 -4.83351918367085
};
\addplot [semithick, color1, mark=*, mark size=2, mark options={solid}]
table {%
0 -0.808930150985718
1 -1.42790022337261
2 -1.9441099768227
3 -2.50730400374563
4 -3.33581683356034
5 -4.36485056142707
};
\addplot [semithick, color2, mark=*, mark size=2, mark options={solid}]
table {%
0 -0.891342713038127
1 -1.59524768370779
2 -2.13713865472894
3 -2.71732308004279
4 -3.53612897851241
5 -4.48949457967658
};
\addplot [semithick, color3, mark=*, mark size=2, mark options={solid}]
table {%
0 -0.800622000694275
1 -1.42384303733826
2 -1.93324059541
3 -2.48570122683876
4 -3.29260485251176
5 -4.28425396942139
};

\nextgroupplot[
width=0.5\textwidth,
height=0.25\textwidth,
axis line style={white!80!black},
legend cell align={left},
legend style={nodes={scale=0.75, transform shape}, fill opacity=0.8, draw opacity=1, text opacity=1, at={(0,1)}, anchor=north west, draw=white!80!black},
tick align=outside,
x grid style={white!80!black},
xlabel={Corruption severity},
xmajorgrids,
xmin=-0.25, xmax=5.25,
xtick style={draw=none},
xtick={-1,0,1,2,3,4,5,6},
xticklabels={\(\displaystyle {−1}\),\(\displaystyle {0}\),\(\displaystyle {1}\),\(\displaystyle {2}\),\(\displaystyle {3}\),\(\displaystyle {4}\),\(\displaystyle {5}\),\(\displaystyle {6}\)},
y grid style={white!80!black},
ylabel={ECE $\downarrow$},
ymajorgrids,
ymin=0, ymax=0.2,
ytick style={draw=none},
ytick={0,0.05,0.1,0.15,0.2},
yticklabels={\(\displaystyle {0.000}\),\(\displaystyle {0.050}\),\(\displaystyle {0.100}\),\(\displaystyle {0.150}\),\(\displaystyle {0.200}\)}
]
\addplot [semithick, color0, mark=*, mark size=2, mark options={solid}]
table {%
0 0.0797670485021048
1 0.0887991289424419
2 0.107940988618828
3 0.1316507648527
4 0.160502427910152
5 0.189601561699852
};
\addlegendentry{MAP}
\addplot [semithick, color1, mark=*, mark size=2, mark options={solid}]
table {%
0 0.0294303081992164
1 0.0148964797990127
2 0.0138620733707806
3 0.0290961382987634
4 0.0549735677333346
5 0.090258854689796
};
\addlegendentry{DE}
\addplot [semithick, color2, mark=*, mark size=2, mark options={solid}]
table {%
0 0.0258117488596692
1 0.0301520232025362
2 0.0267822142453746
3 0.0247793002845965
4 0.0339328417771561
5 0.0640353842435747
};
\addlegendentry{LLLA}
\addplot [semithick, color3, mark=*, mark size=2, mark options={solid}]
table {%
0 0.026262643433652
1 0.0293350123195269
2 0.0223679729987636
3 0.0103615445397431
4 0.0237575595902097
5 0.05934131139767
};
\addlegendentry{MoLA}
\end{groupplot}

\end{tikzpicture}

    \vspace{-1em}
    \caption{Comparison of all methods on ImageNet-C, in terms of the log-likelihood (\textbf{left}, higher is better) and ECE (\textbf{right}, lower is better) metrics. Lines and dots are averages over all corruption types.}
    \label{fig:imagenet-c}
\end{figure*}
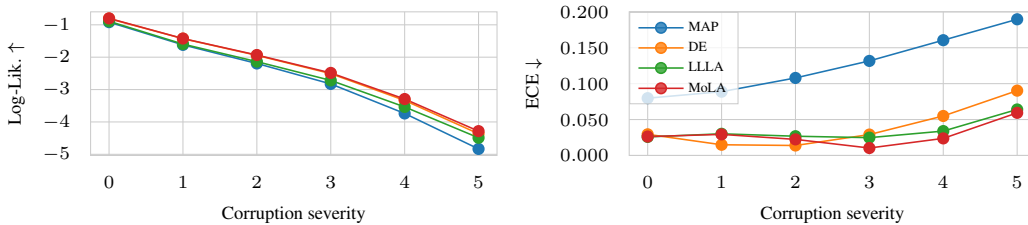

\begin{table}[t]
    \caption{In-distribution: CIFAR-10 $\rightarrow$ OOD. Values are means along with their standard errors
    over five runs with models trained with different random initalizations.
    }
    \label{tab:cifar10_ood}

    \centering
    \fontsize{8}{10}\selectfont

    \begin{sc}
        \begin{tabular}{lccccccc}
            \toprule
            & In-Dist. & \multicolumn{2}{c}{SVHN} & \multicolumn{2}{c}{LSUN} & \multicolumn{2}{c}{CIFAR-100} \\
            Method & MMC & MMC $\downarrow$ & AUROC $\uparrow$ & MMC $\downarrow$ & AUROC $\uparrow$ & MMC $\downarrow$ &  AUROC $\uparrow$ \\
            \midrule
            MAP  &  97.2 $\pm$ 0.0 &  77.5 $\pm$ 2.9 &  91.7 $\pm$ 1.2 &  71.7 $\pm$ 0.8 &  94.3 $\pm$ 0.3 &  79.2 $\pm$ 0.1 &  90.0 $\pm$ 0.1 \\
            DE &  96.1 $\pm$ 0.0 &  62.8 $\pm$ 0.7 &  95.4 $\pm$ 0.2 &  59.2 $\pm$ 0.5 &  96.0 $\pm$ 0.1 &  70.7 $\pm$ 0.1 &  91.3 $\pm$ 0.1 \\
            SWAG & 95.1 $\pm$ 0.4 & 69.3 $\pm$ 4.0 & 91.6 $\pm$ 1.3 &  62.2 $\pm$ 2.3 & 94.0 $\pm$ 0.7 & 73.0 $\pm$ 0.4 & 88.2 $\pm$ 0.5 \\
            MSWAG  &  94.5 $\pm$ 0.2 &  57.0 $\pm$ 1.2 &  95.6 $\pm$ 0.5 &  56.3 $\pm$ 1.0 &  95.6 $\pm$ 0.3 &  65.5 $\pm$ 0.5 &  91.1 $\pm$ 0.1 \\
            LLLA &  94.1 $\pm$ 0.2 &  60.5 $\pm$ 4.0 &  93.6 $\pm$ 1.1 &  54.5 $\pm$ 1.5 &  95.4 $\pm$ 0.3 &  64.8 $\pm$ 0.6 &  90.8 $\pm$ 0.1 \\
            \midrule
            MoLA & 93.8 $\pm$ 0.1 &  \textbf{52.3 $\pm$ 0.7} &  \textbf{96.2 $\pm$ 0.2} &  \textbf{48.3 $\pm$ 0.5} &  \textbf{96.9 $\pm$ 0.1} &  \textbf{61.2 $\pm$ 0.2} &  \textbf{92.0 $\pm$ 0.0} \\
            \bottomrule
        \end{tabular}
    \end{sc}
\end{table}

\subsection{Benchmarks}

\paragraph*{MNIST-R and CIFAR-10-C.}
The results for MNIST-R are in the first row of \Cref{fig:mnist_cifar10}.
MoLA achieves the best results in both LL and ECE metrics, albeit by a small margin to LLLA on LL and to MultiSWAG on ECE. Consistent with \citet{Ovadia2019CanYT}, Deep Ensemble is outperformed by single-DNN Bayesian methods.
On CIFAR-10-C, similar to our observation on MNIST-R, we observe that while all methods' performances degrade as the corruption severity increases, MoLA yields the best results on both metrics across all severity levels (\Cref{fig:mnist_cifar10}, second row). MultiSWAG performs similarly to MoLA for low corruption levels but worsens as the corruption levels increases. While Deep Ensemble improves the vanilla MAP model, it performs worse than MoLA and MultiSWAG. Among the single-model methods, LLLA yields the best results and is only beaten by MoLA and MultiSWAG, as expected.
In contrast to \citet{Ovadia2019CanYT}, a single-model BNN, namely LLLA, outperforms DE. One likely explanation is that in contrast to the last-layer VI method used in \citet{Ovadia2019CanYT}, LLLA does not change the underlying optimization process since it is based on the MAP estimate. This indicates that capturing multiple modes is not the sole reason why DE typically outperforms methods based on only a single model. However, capturing multiple modes further improves performance, as can be seen from the improvement of MoLA compared to LLLA.

\paragraph*{ImageNet-C.}
In \Cref{fig:imagenet-c} we observe that LLLA improves MAP in terms of LL and significantly so in terms of ECE, especially in higher severity levels.
The results for DE and MoLA on LL are similar---though, as expected, both methods perform better than MAP and LLLA.
MoLA achieves better calibration in terms of ECE compared to DE in more challenging scenarios, i.e. at higher severity levels. Nevertheless, DE is more calibrated than LLLA and MoLA for low corruption severity.
One potential explanation for these results is that DE is already well-calibrated in itself and that MoLA is decreasing the confidence, resulting in underconfident predictions for low corruption levels.
These results are reminiscent of the work by \citet{wen2020combining} who observed underconfident predictions as a result of combining DE with a method to improve calibration, in this case, data augmentation.

\paragraph*{OOD detection.}
We present the OOD detection results for CIFAR-10 in \Cref{tab:cifar10_ood}. We observe that MoLA achieves significantly the best results across all OOD test sets and all metrics considered, without sacrificing its in-distribution confidence estimates.\footnote{Significances are established by comparing means and error bars.} Unlike in the previous dataset shift experiments, here we observe that all ensemble methods (Deep Ensemble, MultiSWAG, and MoLA) yield better results than single-model methods (MAP, SWAG, LLLA) in terms of the AUROC metric. This signifies that considering multiple modes is beneficial for discriminating in- against out-of-distribution data. Endowing each of the ensemble members with local uncertainty via MoLA and MultiSWAG further improves this. The MNIST OOD experiment follows a similar trend; we present the full results in \Cref{tab:mnist_ood}.

\subsection{Computational and Memory Cost}
\label{subsec:cost}

We measure the wall-clock time for inference and prediction on the CIFAR-10 test set (cf. \Cref{tab:costs}).
Inference only has to be done once---for LLLA and MoLA it includes the computation and inversion of the K-FAC GGN and the tuning of the prior precision.
We can see that inference is more than an order of magnitude faster for LLLA and MoLA than for SWAG and MultiSWAG respectively.
For prediction, LLLA and MoLA are about as fast as MAP and DE respectively, due to our use of the multi-class probit approximation. Moreover, they only require small additional memory overhead. In contrast, SWAG and MultiSWAG require multiple forward passes which results in slower prediction speed; also, they require multiple snapshots of the model which results in higher memory requirements.

\begin{table}[t]
    \caption{In-distribution: MNIST $\rightarrow$ OOD. Values are means along with their standard errors
    over five random initialization.
    }
    \label{tab:mnist_ood}
    \centering
    \fontsize{8}{10}\selectfont
    \begin{sc}
        \begin{tabular}{lcccccccc}
            \toprule
             & \multicolumn{1}{c}{In-Dist.} & \multicolumn{2}{c}{FMNIST} & \multicolumn{2}{c}{EMNIST} & \multicolumn{2}{c}{KMNIST} \\
             Method & MMC & MMC $\downarrow$ & AUROC $\uparrow$ & MMC $\downarrow$ & AUROC $\uparrow$ & MMC $\downarrow$ &  AUROC $\uparrow$ \\
            \midrule
            MAP &  99.4 $\pm$ 0.0 &  64.1 $\pm$ 0.5 &  99.0 $\pm$ 0.0 &  83.6 $\pm$ 0.3 &  93.7 $\pm$ 0.3 &  77.4 $\pm$ 0.3 &  97.1 $\pm$ 0.1 \\
            DE &  99.2 $\pm$ 0.0 &  55.2 $\pm$ 0.4 &  99.3 $\pm$ 0.0 &  75.8 $\pm$ 0.2 &  95.2 $\pm$ 0.0 &  66.0 $\pm$ 0.3 &  98.4 $\pm$ 0.0 \\
            SWAG &  99.4 $\pm$ 0.0 &  64.6 $\pm$ 0.2 &  99.0 $\pm$ 0.0 &  84.2 $\pm$ 0.2 &  93.6 $\pm$ 0.2 &  78.6 $\pm$ 0.3 &  97.1 $\pm$ 0.1 \\
            MSWAG  &  99.3 $\pm$ 0.0 &  55.6 $\pm$ 0.4 &  99.3 $\pm$ 0.0 &  76.1 $\pm$ 0.3 &  95.2 $\pm$ 0.0 &  66.3 $\pm$ 0.3 &  \textbf{98.5 $\pm$ 0.0} \\
            LLLA &  98.2 $\pm$ 0.0 &  48.2 $\pm$ 0.5 &  99.2 $\pm$ 0.0 &  71.5 $\pm$ 0.3 &  94.4 $\pm$ 0.2 &  63.3 $\pm$ 0.4 &  97.5 $\pm$ 0.1 \\
            \midrule
            MoLA &  98.3 $\pm$ 0.0 &  \textbf{44.0 $\pm$ 0.5} &  \textbf{99.5 $\pm$ 0.0} &  \textbf{67.4 $\pm$ 0.2} &  \textbf{95.7 $\pm$ 0.1} &  \textbf{56.8 $\pm$ 0.4} &  \textbf{98.5 $\pm$ 0.0} \\
            \bottomrule
        \end{tabular}
    \end{sc}
\end{table}

\section{Proofs and Derivations}
\label{sec:proofs}

For completeness, we first derive \Cref{eq:mola_predictive}. Notice that the summation in the definition of $p_\text{MoLA}(\rvf_* | \vx_*, \mathcal{D})$ in \eqref{eq:mola_output_dist} is finite and thus we can safely interchange it with integrals. Now, using \eqref{eq:pred} and \eqref{eq:mola_output_dist}, we have
\begin{equation*}
    \begin{aligned}
        p_\text{{MoLA}}(\vy = \ve_c | \vx_*, \mathcal{D}) &= \int \sigma(\rvf_*)_c \, p_\text{MoLA}(\rvf_* | \vx_*, \mathcal{D}) \,d\rvf_* \\
        &= \int \sigma(\rvf_*)_c \left( \sum_{k=1}^K \pi^{(k)} \, \mathcal{N}(\rvf_* | \vm_*^{(k)}, \mC_*^{(k)}) \right) \,d\rvf_* \\
        &= \sum_{k=1}^K \pi^{(k)} \left( \int \sigma(\rvf_*)_c \, \mathcal{N}(\rvf_* | \vm_*^{(k)}, \mC_*^{(k)}) \,d\rvf_* \right) \\
        &\approx \sum_{k=1}^K \pi^{(k)} \, \sigma(\vz_*^{(k)})_c ,
    \end{aligned}
\end{equation*}
as required, where we have used the linearity of the integral in the third and the multi-class probit approximation \eqref{eq:extended_mackay} in the last step.

\lllmulticlass*
\begin{proof}
    Let $\rvf: \R^{D} \rightarrow \R^C$ be a neural network with activation functions $\text{ReLU}(x) = \max\{0, x\},$ then $\rvf$ is a piecewise affine function \cite{AroraBMM18}. A function $\rvf$ is called piecewise affine if there exists a finite set of polytopes $\{Q_i\}_{i=1}^I$ with $\cup_{i=1}^I Q_i = \R^D$ and $\rvf$ is affine when restricted to any single $Q_i;$ we call $Q_i$ a linear region. The following Lemma, which we adopt without proof from \citet{Hein_2019_CVPR}, says that we can write the neural network as an affine function for data points scaled by a sufficiently large factor.
    \begin{lemma}[\citet{Hein_2019_CVPR}] \label{lemma:linear_regions}
    Let $\{Q_i\}_{i=1}^I$ be the set of linear regions associated with the neural network $\rvf: \R^D \rightarrow \R^C$ with ReLU activations. For any non-zero $\vx \in \R^D$ there exists $\alpha \in \R$ with $\alpha > 0$ and $t \in \{1, \dots, I\}$ such that $\delta \vx \in Q_t$ for all $\delta \geq \alpha.$
    \end{lemma}
    In the following, we make the dependence of $\vz_i$ on the input $\vx_*$ explicit by writing $\vz_i(\vx_*).$ \Cref{lemma:linear_regions} can be used to derive the following statement. 
    \begin{lemma}[\citet{Kristiadi2020BeingBE}] \label{lemma:increasing}
    For any non-zero $\vx_* \in \R^D$ there exists $\alpha \in \R$ with $\alpha > 0,$ such that $\delta \vx_* \in R$ for all $\delta \geq \alpha.$ Also, we assume that the neural network $\rvf$ has no bias parameters. Then the restriction $|\vz|_R(\delta \vx_*)_i|$ is an increasing function in $\delta$ for all $i \in \{1, \dots, C\}.$ 
    \end{lemma}
    To reiterate, the multi-class probit approximation only uses the diagonal elements $\mC_{*ii}$ of the output covariance matrix $\mC_*$ and hence, without loss of generality, for our analysis we only need to consider the posterior distributions $\mathcal{N}(\vw_{(i)} | \vmu_{(i)}, \mSigma_{(i)})$ over each row $\vw_{(i)}$ of the weight matrix $\mW$, instead of the full posterior over $\mathrm{vec}(\mW)$.
    With this, \Cref{lemma:linear_regions} can be used to derive the following bound.
    \begin{lemma}[\citet{Kristiadi2020BeingBE}] \label{lemma:z_bound}
    For any non-zero $\vx_* \in \R^D$ there exists $\alpha \in \R$ with $\alpha > 0,$ such that $\delta \vx_* \in R$ for all $\delta \geq \alpha.$ We have
    \begin{equation} \label{eq:z_bound}
        \lim_{\delta \rightarrow \infty} |\vz|_R (\delta \vx_*)_i | \leq \frac{||\vmu_{(i)}||_2}{\sqrt{\pi / 8 \, \lambda_\text{min}(\mSigma_{(i)})}} =: b_i,
    \end{equation}
    for all $i \in \{1, \dots, C\}.$
    \end{lemma}
    Using these results, we can now proof the desired statement. Let $\vx_* \in \R^D$ be arbitrary but non-zero. Then there exists $\alpha \in \R$ with $\alpha > 0,$ such that $\delta \vx_* \in R$ for all $\delta \geq \alpha.$ Using the multi-class probit approximation \eqref{eq:extended_mackay}, we have for all $\delta \geq \alpha$
    \begin{equation*}
        \begin{aligned}
            p_\text{BNN}(\vy = \ve_{c_*} | \delta\vx_*, \mathcal{D}) &\approx \sigma(\vz(\delta \vx_*))_{c_*} \\
            &= \frac{1}{1 + \sum_{i \neq c_*} \exp(\vz(\delta \vx_*)_i - \vz(\delta \vx_*)_{c_*})} \\
            &\leq \frac{1}{1 + \sum_{i \neq c_*} \exp(-(|\vz(\delta \vx_*)_i| + |\vz(\delta \vx_*)_{c_*}|))} \\
            &\leq \lim_{\delta \rightarrow \infty} \frac{1}{1 + \sum_{i \neq c_*} \exp(-(|\vz(\delta \vx_*)_i| + |\vz(\delta \vx_*)_{c_*}|))} \\
            &= \frac{1}{1 + \sum_{i \neq c_*} \exp(- \lim_{\delta \rightarrow \infty} (|\vz(\delta \vx_*)_i| + |\vz(\delta \vx_*)_{c_*}|))} \\
            &\leq \frac{1}{1 + \sum_{i \neq c_*} \exp(- (b_i + b_{c_*}))} ,
        \end{aligned}
    \end{equation*}
    where we have used \Cref{lemma:increasing} in the second and \Cref{lemma:z_bound} in the last inequality. This concludes the proof.
\end{proof}

\molaconfidence*
\begin{proof}
    As in \Cref{lemma:lll_multiclass}, we only need to consider the posterior distributions
    \begin{equation*}
        p_\text{MoLA}(\vw_{(i)} | \mathcal{D}) = \sum_{k=1}^K \pi^{(k)} \mathcal{N}(\vw_{(i)} | \vmu_{(i)}^{(k)}, \mSigma_{(i)}^{(k)})
    \end{equation*}
    over each row $\vw_{(i)}$ of the weight matrix $\mW$, instead of the full MoLA posterior in \eqref{eq:mola_posterior}.
    Using the approximation from \Cref{eq:mola_predictive}, we have
    \begin{equation*}
    \begin{aligned}
    p_{\emph{MoLA}}(\vy = \ve_{c_*} | \delta \vx_*, \mathcal{D}) &\approx \sum_{k=1}^K \pi^{(k)} \, \sigma(\vz^{(k)}(\delta \vx_*))_{c_*} \\
    &\leq \sum_{k=1}^K \frac{\pi^{(k)}}{1 + \sum_{i \neq c_*} \exp(- (b_i^{(k)} + b_{c_*}^{(k)}))},
    \end{aligned}
    \end{equation*}
    where we have applied \Cref{lemma:lll_multiclass} to each mixture component to get the inequality.
\end{proof}

\subsection{Mixture Weights via Model Selection}
\label{sec:marglik_weights}

The goal is to not only infer the neural network's last-layer weight matrix $\mW,$ but also include model selection into the inference problem. Hence, consider the models $\mathcal{M}_k$ of the $K$ components, where the model typically consists of the DNN architecture and hyperparameters. We now want to infer the joint posterior
\begin{equation}
\begin{aligned}
    p(\mW, \mathcal{M}|\mathcal{D}) &= \frac{p(\mathcal{D}|\mW, \mathcal{M}) p(\mW|\mathcal{M}) p(\mathcal{M})}{p(\mathcal{D})} \\
    &= \frac{p(\mW | \mathcal{D}, \mathcal{M}) p(\mathcal{D}|\mathcal{M}) p(\mW|\mathcal{M}) p(\mathcal{M})}{p(\mW | \mathcal{M}) p(\mathcal{D})} \\
    &= p(\mW | \mathcal{D}, \mathcal{M}) p(\mathcal{M} | \mathcal{D}). \\
\end{aligned}
\end{equation}
We can see that $p(\mW | \mathcal{D}, \mathcal{M})$ is our regular posterior over the weights, given a model $\mathcal{M}.$ We consider the model to have a categorical distribution with $K$ categories and impose a uniform prior on them, i.e. $p(\mathcal{M}) = \mathcal{U}(1, \dots, K).$ The posterior probability of the $k$th model $\mathcal{M}_k$ is then given by
\begin{equation}
\begin{aligned}
    p(\mathcal{M}_k | \mathcal{D}) &= \frac{p(\mathcal{D} | \mathcal{M}_k) p(\mathcal{M}_k)}{p(\mathcal{D})} \\ 
    &= \frac{p(\mathcal{D} | \mathcal{M}_k) p(\mathcal{M}_k)}{\sum_{k^{'}=1}^K p(\mathcal{D} | \mathcal{M}_{k^{'}}) p(\mathcal{M}_{k^{'}})} \\ 
    &= \frac{p(\mathcal{D} | \mathcal{M}_k) \frac{1}{K}}{\sum_{k^{'}=1}^K p(\mathcal{D} | \mathcal{M}_{k^{'}}) \frac{1}{K}} \\ 
    &= \frac{p(\mathcal{D} | \mathcal{M}_k)}{\sum_{k^{'}=1}^K p(\mathcal{D} | \mathcal{M}_{k^{'}})},
\end{aligned}
\end{equation}
where we can recognize $p(\mathcal{D} | \mathcal{M}_k)$ as the marginal likelihood corresponding to the posterior over the weights $\mW$ given the model $\mathcal{M}_k.$
Using the posterior over weights and models, the predictive distribution under the Laplace approximation becomes
\begin{equation}
\begin{aligned}
    p(\vy = \ve_c | \vx_{*}, \mathcal{D}) &= \int \sum_{k=1}^K p(\vy = \ve_c |\vx_{*}, \mW, \mathcal{M}_k) p(\mW | \mathcal{D}, \mathcal{M}_k) p(\mathcal{M}_k | \mathcal{D}) \,d \mW \\
    &= \sum_{k=1}^K \underbrace{p(\mathcal{M}_k | \mathcal{D})}_{=: \pi^{(k)}} \int p(\vy = \ve_c |\vx_{*}, \mW, \mathcal{M}_k) p(\mW | \mathcal{D}, \mathcal{M}_k) \,d \mW \\
    &= \sum_{k=1}^K \pi^{(k)} \int \sigma(\rvf_*)_c \, p(\rvf_*|\mathcal{D}, \mathcal{M}_k) \, d\rvf_* \\
    &= p_\text{{MoLA}}(\vy = \ve_c | \vx_*, \mathcal{D}),
\end{aligned}
\end{equation}
where we have used the predictive distribution of the LLLA from \Cref{eq:pred} in the third step. The mixture weights are therefore the normalized marginal likelihoods of the $K$ different models.
Since we apply the Laplace approximation to each DNN, each model's weight posterior has Gaussian form, and therefore, the marginal likelihood of each model $\mathcal{M}_k$ can be estimated in closed form.

\section{Experimental Details}
\label{sec:details}

\subsection{Algorithm}
\label{sec:algorithm}

We present the algorithm for inference and prediction with MoLA using LLLA and a K-FAC approximation of the Hessian in \Cref{alg:mola}. Note that the inference function only needs to be called once. The algorithm for LLLA with K-FAC is adopted from \citet{Kristiadi2020BeingBE}.

\begin{algorithm}[tb]
   \caption{MoLA inference and prediction functions with LLLA and K-FAC approximation of the Hessian.}
   \label{alg:mola}
\begin{algorithmic}
   \FUNCTION{Inference($\{ \rvf^{(k)} \}_{k=1}^K$, $\mathcal{D}_{\text{train}}, \mathcal{D}_{\text{val}}$)}
   \STATE {\bfseries Input:} Set of $K$ pre-trained neural networks $\rvf(\cdot)^{(k)} = \mW_{\text{MAP}}^{(k)} \vphi(\cdot)^{(k)}$, training dataset $\mathcal{D}_{\text{train}},$ validation dataset $\mathcal{D}_{\text{val}},$ mini-batch size $n,$ running average parameter $\beta$
   \STATE Initialize $\mathcal{A} = \emptyset$ and $\mathcal{B} = \emptyset$.
   \FOR{$k = 1$ {\bfseries to} $K$}
   \STATE Initialize $\mA^{(k)} = \mathbf{0} \in \R^{C \times C}$ and $\mB^{(k)} = \mathbf{0} \in \R^{P \times P}$.
   \FOR{$\mX, \mY$ {\bfseries in} $\text{sampleMiniBatch}(\mathcal{D}_{\text{train}}, n)$}
   \STATE $\hat{\mA}^{(k)}, \hat{\mB}^{(k)} = \text{computeK-FAC}(\mathcal{L}(\rvf(\mX)^{(k)}, \mY), \mW_{\text{MAP}})$
   \STATE $\mA^{(k)} = \beta \mA^{(k)} + (1-\beta) \hat{\mA}^{(k)}$
   \STATE $\mB^{(k)} = \beta \mB^{(k)} + (1-\beta) \hat{\mB}^{(k)}$
   \ENDFOR
   \STATE $\mathcal{A} = \mathcal{A} \cup \{\mA^{(k)}\}$
   \STATE $\mathcal{B} = \mathcal{B} \cup \{\mB^{(k)}\}$
   \ENDFOR
   \STATE $\boldsymbol{\pi} = \text{chooseMixtureWeights}(\cdot)$
   \STATE $\lambda = \text{tunePriorPrecision}(\{ \rvf^{(k)} \}_{k=1}^K, \mathcal{A}, \mathcal{B}, \boldsymbol{\pi}, \mathcal{D}_{\text{val}})$
   \STATE Initialize $\mathcal{U} = \emptyset$ and $\mathcal{V} = \emptyset$.
   \FOR{$k = 1$ {\bfseries to} $K$}
   \STATE $\mU^{(k)} = (\sqrt{|\mathcal{D}_{\text{train}}|} \mA^{(k)} + \sqrt{\lambda} \mathbf{I})^{-1}$
   \STATE $\mV^{(k)} = (\sqrt{|\mathcal{D}_{\text{train}}|} \mB^{(k)} + \sqrt{\lambda} \mathbf{I})^{-1}$
   \STATE $\mathcal{U} = \mathcal{U} \cup \{\mU^{(k)}\}$
   \STATE $\mathcal{V} = \mathcal{V} \cup \{\mV^{(k)}\}$
   \ENDFOR
   \STATE {\bfseries Output:} $\mathcal{U}, \mathcal{V}, \boldsymbol{\pi}$
   \ENDFUNCTION
   \\
   \FUNCTION{Prediction($\{ \rvf^{(k)} \}_{k=1}^K, \mathcal{U}, \mathcal{V}, \boldsymbol{\pi}, \mathcal{D}_{\text{test}}$)}
   \STATE {\bfseries Input:} Set of $K$ pre-trained neural networks $\rvf(\cdot)^{(k)} = \mW_{\text{MAP}}^{(k)} \vphi(\cdot)^{(k)}$, sets $\mathcal{U}$ and $\mathcal{V}$ with the corresponding Kronecker factored covariances, mixture weights $\boldsymbol{\pi} = \{\pi^{(k)}\}_{k=1}^K,$ test dataset $\mathcal{D}_{\text{test}}$
   \STATE Initialize $\mathcal{Y} = \emptyset.$
   \FOR{$\vx_*$ {\bfseries in} $\mathcal{D}_{\text{test}}$}
   \STATE Initialize $\vy = \mathbf{0} \in \R^{C}.$
   \FOR{$k = 1$ {\bfseries to} $K$}
   \STATE $\vm_*^{(k)} = \mW_{\text{MAP}}^{(k)} \vphi_*^{(k)}$
   \STATE $\mC_*^{(k)} = \langle \vphi_*^{(k)}, \mV^{(k)} \vphi_*^{(k)} \rangle \, \mU^{(k)}$
   \STATE $\vz_*^{(k)} = \vm_*^{(k)} / \sqrt{1 + (\pi/8) \, \text{diag}(\mC_*^{(k)})}$
   \STATE $\vy = \vy + \pi^{(k)} \sigmoid(\vz_*^{(k)})$
   \ENDFOR
   \STATE $\mathcal{Y} = \mathcal{Y} \cup \{\vy\}$
   \ENDFOR
   \STATE {\bfseries Output:} $\mathcal{Y}$
   \ENDFUNCTION
\end{algorithmic}
\end{algorithm}

\subsection{Datasets}
\label{sec:data}

The rotated MNIST benchmark uses the standard MNIST dataset \cite{lecun1989} which consists of black and white images of handwritten digits, labeled with 10 classes. The images are roatated by increasing angle up to 180 degrees.
The corrupted CIFAR-10 benchmark consists of the the standard CIFAR-10 dataset and 5 corruption levels with 16 different corruptions types each. Examples for the corruption types are Gaussian noise and changed brightness. The corrupted ImageNet benchmark consists of the standard ImageNet2012 dataset and 5 corruption levels with 19 corruption types each. The types of corruptions are the same as for corrupted CIFAR-10. Both corrupted datasets were proposed in \citet{hendrycks2018benchmarking}; see Figure 1 in their paper for example images of the different corruption types.

As OOD datasets for the OOD experiment we use SVHN, LSUN-Classroom, and CIFAR-100. All datasets are available through the PyTorch \texttt{torchvision} package \cite{pytorch}.

\subsection{Training}
\label{sec:training}

For MNIST we use a five-layer LeNet \cite{lecun1989} and Adam \cite{Kingma2015adam} and initial learning rate of 1e-3. For CIFAR-10 we use a Wide ResNet (WRN) 16-4 \cite{Zagoruyko2016WRN} with dropout rate of 0.3 and stochastic gradient descent (SGD) with Nesterov momentum of 0.9 and inital learning rate of 0.1; we also use standard data augmentation, i.e. random cropping and flipping. For both, we train for 100 epochs with mini-batch size of 128, use a cosine learning rate schedule \cite{LoshchilovH17}, and weight decay with factor 5e-4. For ImageNet we use a WRN 50-2 and follow the PyTorch example training script
\footnote{\url{https://github.com/pytorch/examples/tree/master/imagenet}}: we optimize for 90 epochs with SGD with initial learning rate of 0.1 and decrease it by 90\% at epoch 30 and epoch 60; we also use weight decay of 1e-4 and random cropping and flipping as data augmentation.
The LeNet, WRN 16-14, and WRN 50-2 achieve 99.2\%, 94.8\%, and 77.6\% accuracy, respectively.

\subsection{Hyperparameter Tuning}
\label{sec:tuning}

We only tune one hyperparameter for LLLA and MoLA, namely the prior precision. For MoLA we tune a single scalar prior precision for all components. We want to choose it high enough to avoid underconfidence on in-distribution data while keeping it low enough that the behaviour does not becomes too similar to MAP or DE. Hence, we perform a simple grid search on a validation set of size 2000, going over a range of up to 100 values for the prior precision. The validation set is a randomly sampled subset of the test dataset. We start with a prior precision of 1e-4 and go up to at most 1e3 and check if the mean confidence and Brier score on the validation set fulfill a threshold. We then choose the first, i.e. smallest, prior precision which fulfills these thresholds. We set the thresholds for the mean confidence to be slightly below the accuracy of MAP on the validation set. For MNIST the threshold is 0.98 and for CIFAR-10 it is 0.94. While we have also set a threshold on the Brier score for our experiments, just the threshold on the average confidence is sufficient by itself. Note that we could also choose some other method to tune the prior precision; however, in practice our heuristic method seems sufficient. 

For SWAG and MultiSWAG, we follow \citet{maddox2019simple} and run stochastic gradient descent with a constant learning rate on the pre-trained models to collect one model snapshot per epoch, for a total of 40 snapshots. At test time, we then make predictions by using 30 Monte Carlo samples from the posterior distribution; we correct the batch normalization statistics of each sample as described in \citet{maddox2019simple}.
To tune the constant learning rate, we used the same approach as for tuning the prior precision of LLLA and MoLA described above, combining a grid search with a threshold on the mean confidence.
For MNIST, we defined the grid to be the set $\{$ 1e-1, 5e-2, 1e-2, 5e-3, 1e-3 $\}$, yielding an optimal value of 1e-2.
For CIFAR-10, searching over the same grid suggested that the optimal value lies between 5e-3 and 1e-3; another, finer-grained grid search over the set $\{$ 5e-3, 4e-3, 3e-3, 2e-3, 1e-3 $\}$ then revealed the best value to be 2e-3.

\subsection{Implementation Details of MIMO and MIMO-MoLA}

For MIMO and MIMO-MoLA, we only use $K=3$ ensemble components  due to the limited capacity of the model. We use input repetition $\rho=0$ and batch repetition of $4;$ see \citet{Havasi2021mimo} for more details on these hyperparameters. As for MoLA, we only tune a single scalar prior precision for MIMO-MoLA, using the heuristic described in \Cref{sec:tuning}. In contrast to LLLA and MoLA, we use the full GGN of each head of the last layer for MIMO-MoLA, in favor of implementational convenience. While the results on corrupted CIFAR-10 are promising, a more thorough study of MIMO-MoLA is necessary.

\subsection{Computing Resources}
\label{sec:compute}

All experiments were conducted on a cluster with multiple NVIDIA RTX 2080Ti and Tesla V100 GPUs.

\section{Additional Results}
\label{sec:add_results}

\subsection{Varying the Number of Mixture Components}
\label{sec:nr_comps}

To check the behavior of MoLA with increasing number of mixture components, we compare DE and MoLA with one to ten mixture components on CIFAR-10-C. For each number of mixture components, we average each metric over all corruption levels and types.
We observe that MoLA follows a similar power-law-like behaviour as DE with increasing number of mixture components; see \citet{lobacheva2020power} for a detailed investigation in these power laws in DEs.
Interestingly, just one component of MoLA, i.e. LLLA, performs better than DE with ten components regarding log-likelihood and ECE; it also achieves lower mean confidence. The only metric where DE performs as well as MoLA is accuracy; DE is even better, albeit by a very small margin (note the scale of the y-axis).

\section{Full Results}
\label{sec:full_results}

\subsection{Variations of LLLA}
\label{sec:llla_full}

To explore the effect of the K-FAC approximation of the Hessian (\Cref{subsec:practical}) and the multi-class probit approximation (MPA) to the predictive distribution on performance, we compare different variations of LLLA on CIFAR-10-C. We test all six combinations of diagonal/K-FAC approximation/full Hessian and the MC integral \eqref{eq:mc_integral}/MPA \eqref{eq:extended_mackay}. We use 100 MC samples for the MC integral. The prior precision is tuned by choosing the smallest from a grid with $100$ steps which leads to the mean confidence on the validation set being larger than $0.94.$
Regarding accuracy and Brier score, the least accurate approximation, i.e. the diagonal/MPA combination, performs worst.
The most accurate approximation, i.e. the full/MC combination performs best regarding log-likelihood (by a small margin and with overlapping error bars), ECE (tied with K-FAC/MC), and Brier score. 
While we prefer the combination of K-FAC, which is even feasible for models with a large last-layer, e.g. WRN 50-2 for ImageNet, and the MPA, which enables predictions with negligible additional cost compared to a regular forward pass, for our experiments, it is of course feasible to use a K-FAC/full covariance and a MC approximation to the predictive distribution. Using one of these two combinations might further increase performance.

\subsection{Rotated MNST}
\label{sec:r_mnist_full}

On MNIST-R, MoLA is only marginally better than the next best methods regarding log-likelihood, Brier score, and ECE. In terms of accuracy, all methods are very similar, with MultiSWAG being the best by small margin. MoLA achieves the lowest mean confidence. There is no clear trend regarding MCE.

\subsection{Corrupted CIFAR-10}
\label{sec:c_cifar_full}

On CIFAR-10-C, MoLA performs best regarding log-likelihood, Brier score, and ECE. In terms of accuracy it is similar to MultiSWAG, altough slightly better. There is no clear trend regarding MCE. Interestingly, LLLA performs better than all methods besides MoLA regarding log-likelihood, despite only using a single model; moreover, it outperforms DE regarding ECE.

\subsection{Corrupted ImageNet}
\label{sec:c_imagenet_full}

On ImageNet-C, MoLA performs more or less the same as DE regarding log-likelihood, accuracy, and Brier score. MoLA performs worse than DE on ECE for low corruption severity but better for high corruption severity. DE is alreay well calibrated for low corrutpion severity and MoLA becomes slightly underconfident. LLLA outperforms MAP on log-likelihood and Brier score, but only for higher corruption severity. On ECE, LLLA performs much better than MAP and almost the same as DE and MoLA (even better than DE for high corruption severity). In terms of accuracy, LLLA performs about the same as MAP. Regarding MCE, MAP is outperformed by all other methods.

\begin{figure*}[ht]
\vskip 0.2in
\begin{center}
\centerline{\includegraphics[width=\textwidth]{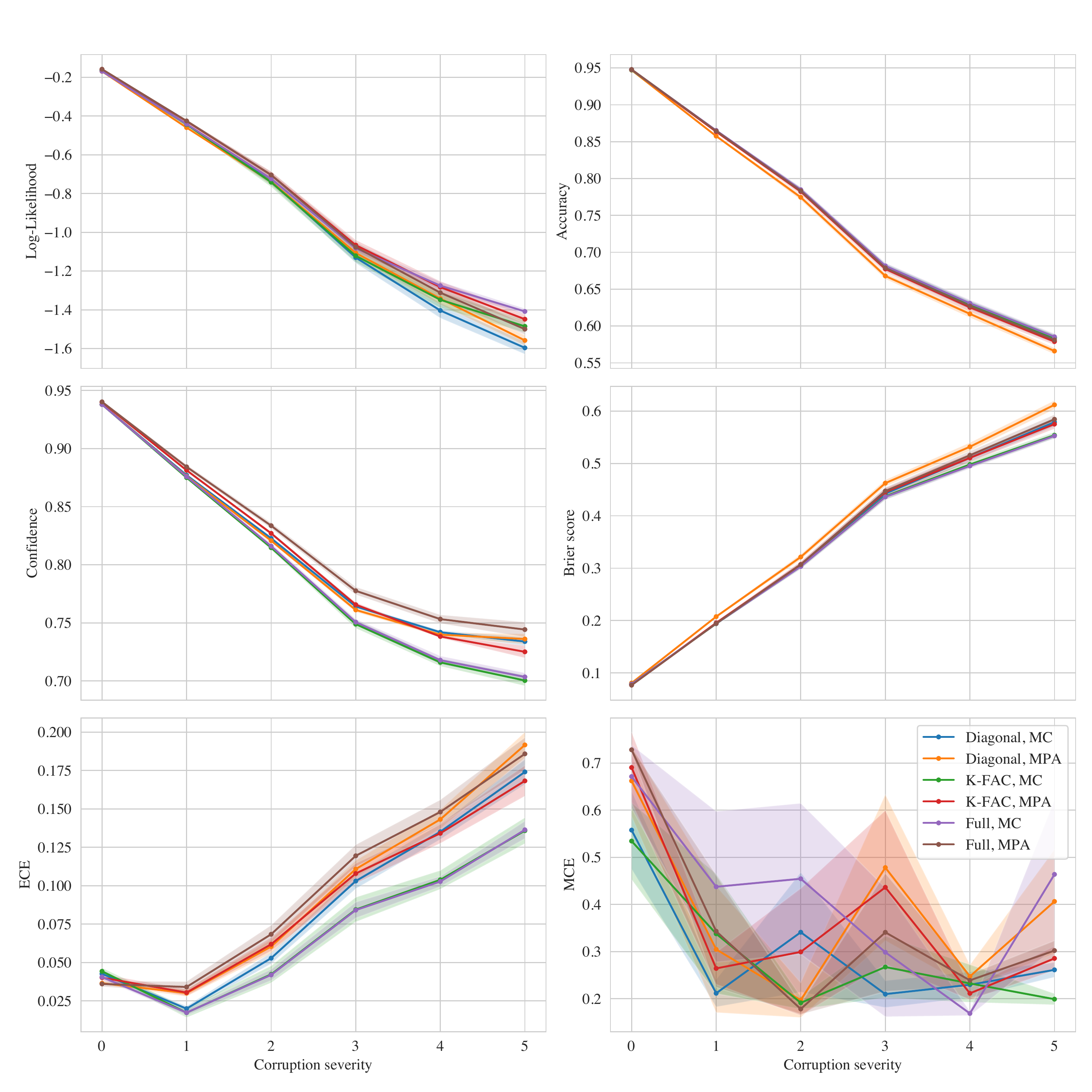}}
\caption{Comparison of different variations of LLLA on the corrupted CIFAR-10 dataset. Lines and dots are means over all corruption types at a particular severity level while shades represent standard errors of five runs. MC stands for ``Monte Carlo'' and MPA for ``multi-class probit approximation''.}
\label{llla_full}
\end{center}
\vskip -0.2in
\end{figure*}

\begin{figure*}[ht]
\vskip 0.2in
\begin{center}
\centerline{\includegraphics[width=\textwidth]{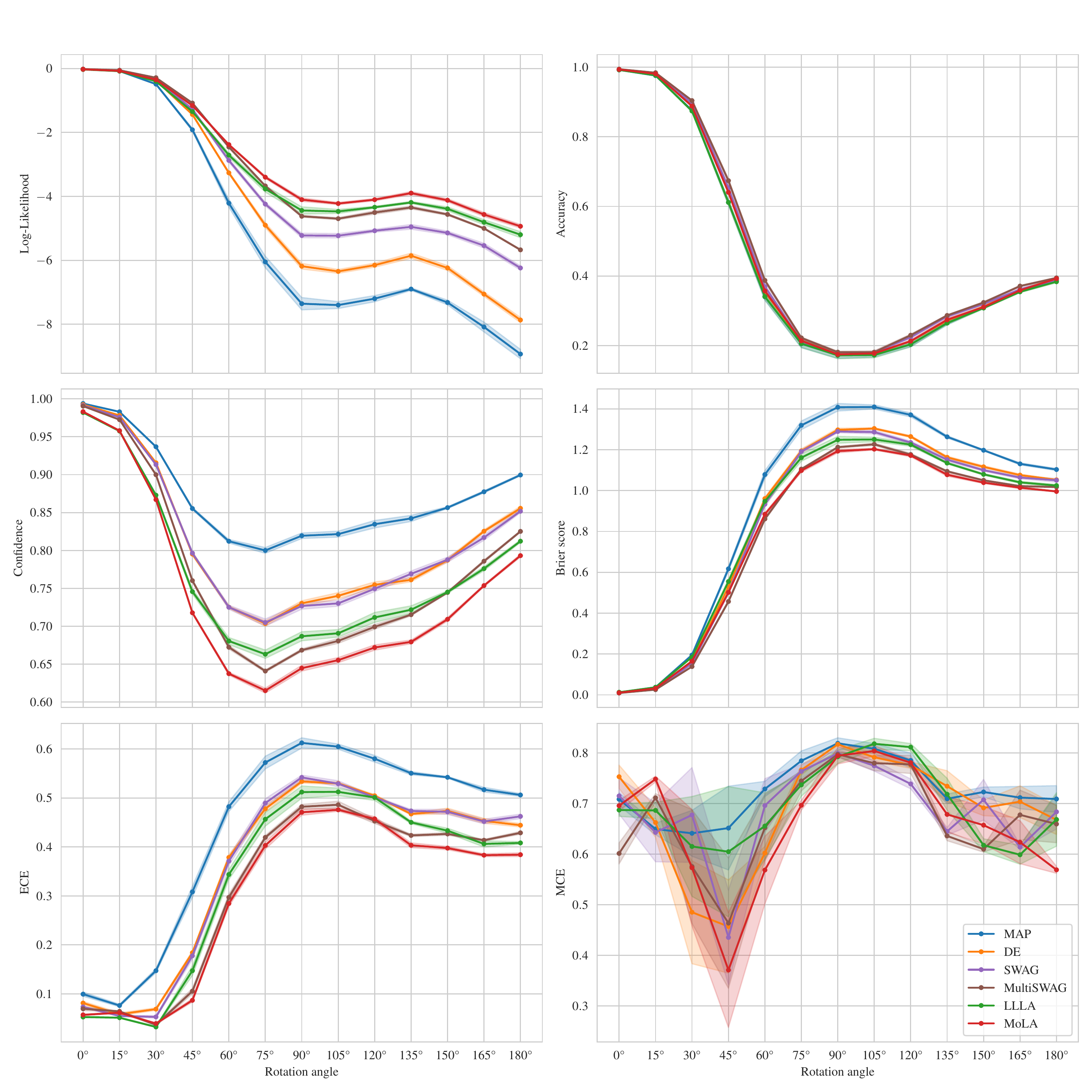}}
\caption{Comparison of all methods on rotated MNIST. Metrics are averaged over all corruption types for each corruption severity. Lines and dots represent means while shades represent standard errors of five runs with models trained with different random initalizations.}
\label{r_mnist_full}
\end{center}
\vskip -0.2in
\end{figure*}

\begin{figure*}[ht]
\vskip 0.2in
\begin{center}
\centerline{\includegraphics[width=\textwidth]{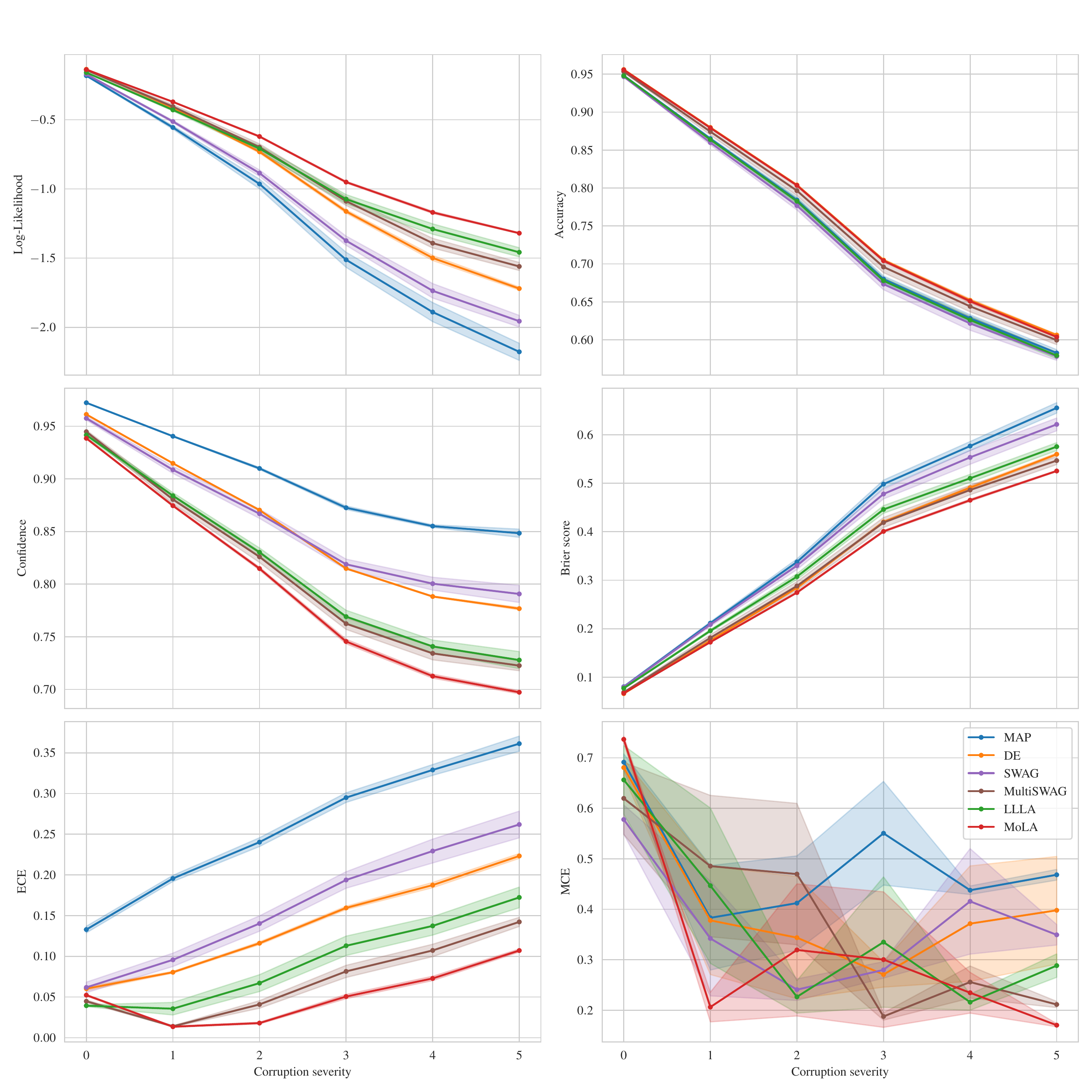}}
\caption{Comparison of all methods on corrupted CIFAR-10. Metrics are averaged over all corruption types for each corruption severity. Lines and dots represent means while shades represent standard errors of five runs with models trained with different random initalizations.}
\label{c_cifar10_full}
\end{center}
\vskip -0.2in
\end{figure*}

\begin{figure*}[ht]
\vskip 0.2in
\begin{center}
\centerline{\includegraphics[width=\textwidth]{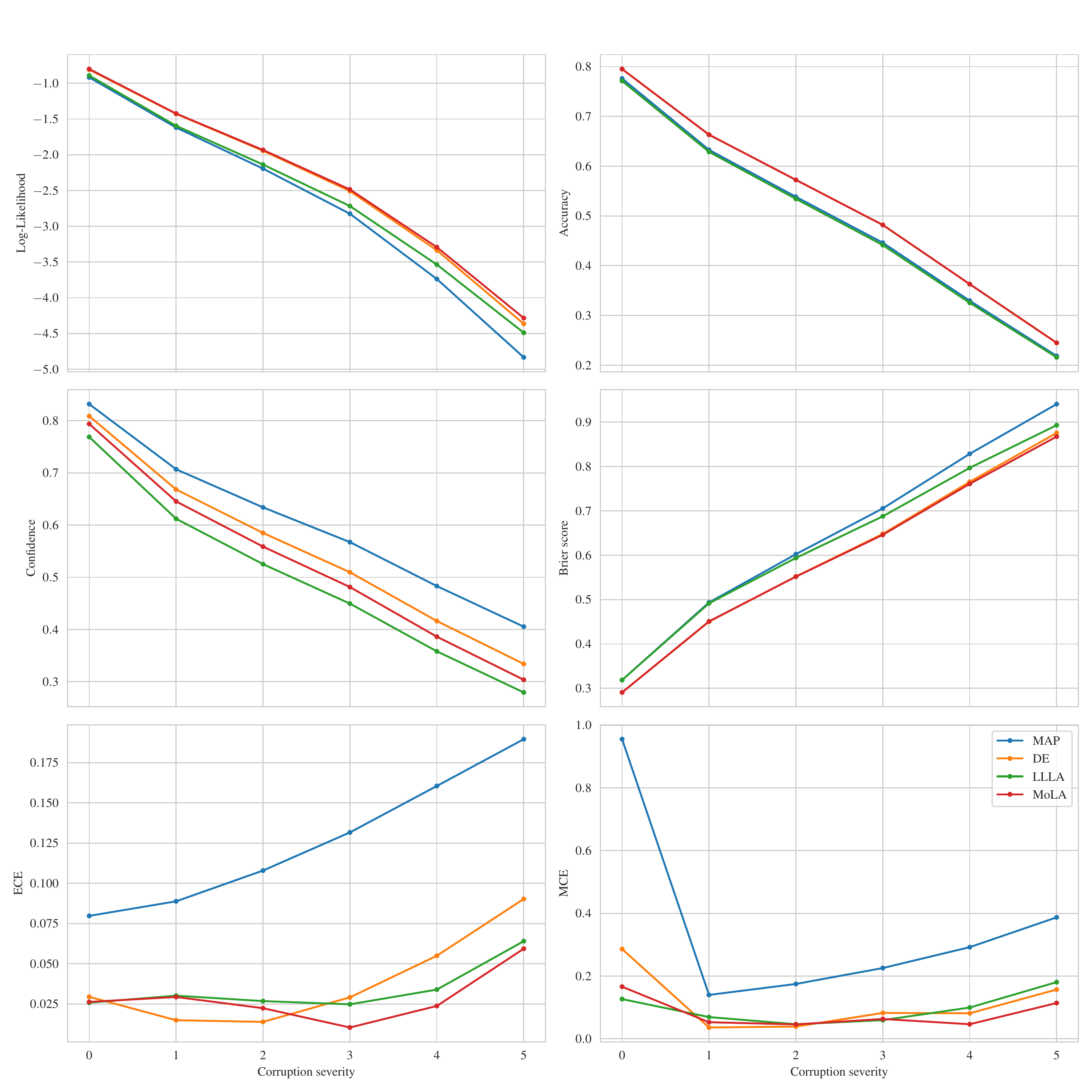}}
\caption{Comparison of MAP, DE, LLLA, and MoLA on corrupted ImageNet. Lines and dots represent the mean over all corruption types for each corruption severity.}
\label{c_imagenet_full}
\end{center}
\vskip -0.2in
\end{figure*}

\begin{figure*}[ht]
\vskip 0.2in
\begin{center}
\centerline{\includegraphics[width=\textwidth]{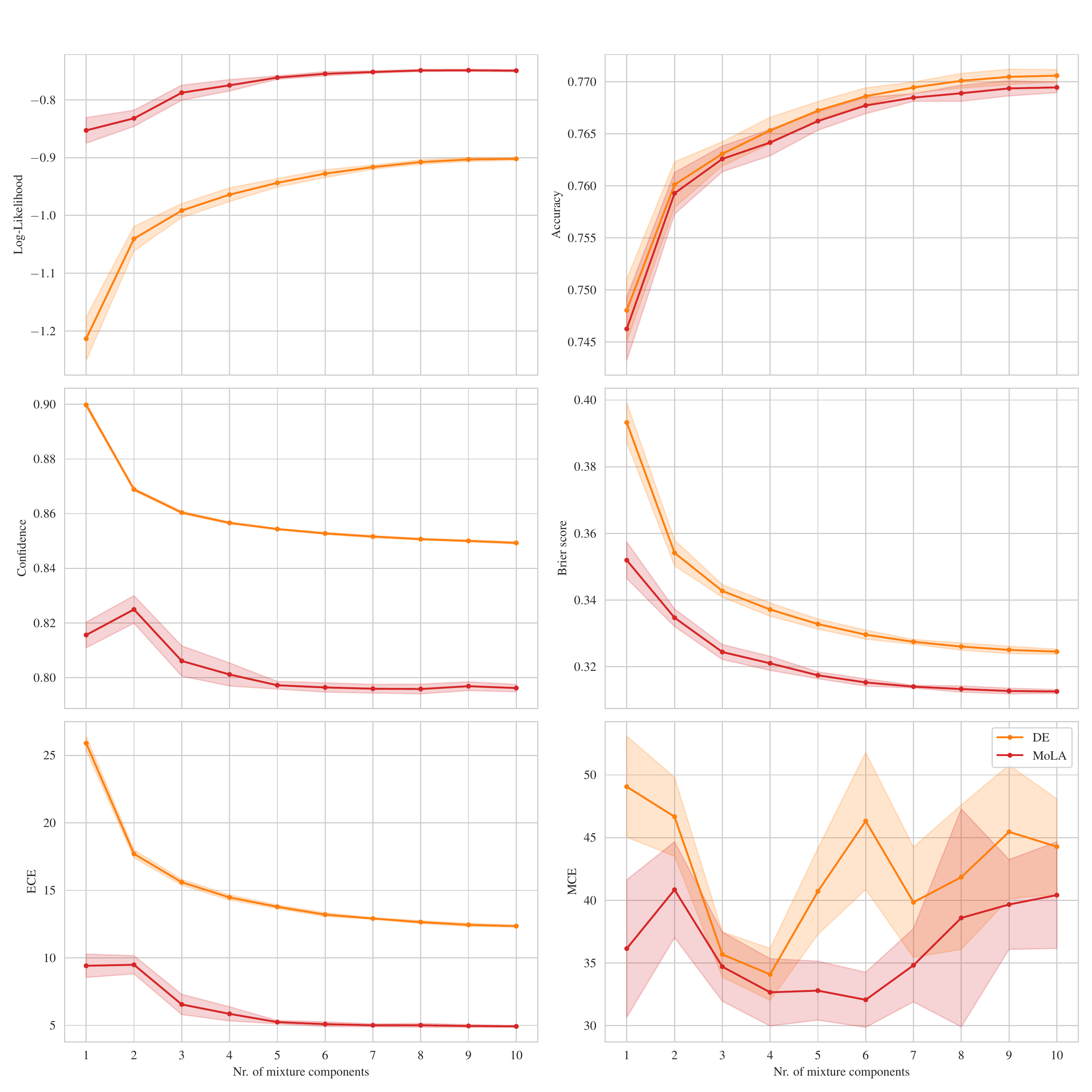}}
\caption{Comparison of different numbers of mixture components on corrupted CIFAR-10. Metrics are averaged over all corruption severities and types. Lines and dots represent means while shades represent standard error of five runs with models trained with different random initalizations.}
\label{nr_comps}
\end{center}
\vskip -0.2in
\end{figure*}

\end{document}